%% file: main.tex
\documentclass[11pt,a4paper]{article}
\usepackage{graphicx}

\usepackage{todonotes}

\usepackage[T1]{fontenc}       
\usepackage[utf8]{inputenc}    

\usepackage{microtype}         

\usepackage[margin=2.3cm]{geometry}  

\usepackage{amsmath}
\usepackage{mathtools}         
\usepackage{amssymb}           
\usepackage{amsthm}            
\usepackage{bm}                
\usepackage{mathrsfs}          
\usepackage{dsfont}   
\usepackage[normalem]{ulem}

\usepackage{xcolor}            
\usepackage{float}             

\usepackage{tikz}
\usetikzlibrary{arrows,calc,positioning,external}

\usepackage[colorlinks, linkcolor=black, citecolor=black, urlcolor=black]{hyperref}

\usepackage[nameinlink,noabbrev]{cleveref}

\usepackage{caption}
\usepackage{subcaption}
\usepackage{wrapfig}

\usepackage[shortlabels]{enumitem}

\usepackage{authblk}

\usepackage[backend=biber, style=alphabetic, maxbibnames=40, doi=false, isbn=false, url=false, eprint=false]{biblatex}
\input{commands.tex}

\title{Dimension-free regimes for unadjusted Langevin in ridge-separable models: A subspace robust Wasserstein analysis} 
\author[1]{Yuansi Chen}
\author[1]{Francesco Pedrotti}
\author[2]{Peter A. Whalley}

\affil[1]{Seminar for Statistics, ETH Zürich}
\affil[2]{Department of Statistics, University of Warwick}

\begin{document}
	
	\maketitle
	
	\begin{abstract}    
		The Unadjusted Langevin Algorithm (ULA) is widely used for high-dimensional Bayesian inference. Classical analyses in Wasserstein-2 distance predict an asymptotic bias and a mixing time that grow with dimension, but this metric is stronger than necessary for many statistical tasks. We instead analyze the bias in subspace robust Wasserstein distance, which controls the mean and arbitrary low-dimensional marginals and is therefore well suited to observables typically estimated in practice. For ridge-separable potentials formed with i.i.d. data, we prove high-probability bounds showing a bell-shaped mixing profile as the number of data points $N$ varies relative to the dimension $d$. In the overparameterized regime, the effective dimension is the rank $N$ of the data-dependent perturbation, leading to complexity depending on $N$ rather than $d$. In the large-sample regime, Gaussian-perturbation estimates yield dimension-free complexity once $N$ is sufficiently large. In contrast, near the proportional regime $\asymp d$, we prove that dimension-free guarantees are impossible in general. 
	\end{abstract}

	\section{Introduction}

	A fundamental problem in high-dimensional statistics, machine learning and Bayesian inference is to estimate the expectation of a functional with respect to a probability measure \cite{Roberts2004,liu2001monte}.
	An important class of examples \cite{berkson1944application,pampel2020logistic,mou2021high, freund2022convergence, chen2023does} is given by probability measures of the form  $\pi \propto e^{-V}$, where the potential $V\colon \R^d \to \R$ takes the form 
	\begin{equation}
		\label{eq:potential-stats-intro}
		V(x) = \frac{\lambda}{2} \abs*{x}^2 +\sum_{i=1}^N u_i\tond*{\scal*{x,a_i}},
	\end{equation}
    and which we refer to as \emph{ridge-separable potentials}.
	For example, in Bayesian posteriors, the first term corresponds to a Gaussian prior component, with precision  $\lambda >0$, while the vectors $(a_i)_{i=1}^N$ in the second term represent data-dependent directions drawn from a distribution in $\R^d$, and $(u_i)_{i=1}^N$ are one-dimensional losses or negative log-likelihoods. 
	Such posteriors arise, for example, in Bayesian ridge regression and Bayesian generalized linear models, including logistic, probit and Poisson regression \cite{gelman1995bayesian}, as well as in single-neuron Bayesian neural networks \cite{Neal}.

	A standard approach to approximate expectations under $\pi$ is Markov chain Monte Carlo (MCMC). The underlying idea is to design a Markov chain $(X_k)_k$ which is ergodic with respect to $\pi$ (or an accurate approximation), and simulate it for a long enough time. 
	Then, under standard ergodicity assumptions, the empirical average of the test function $f$ along $(X_k)_k$ accurately approximates its expectation under $\pi$.
	
	One simple approach for differentiable $V$ is the \emph{Unadjusted Langevin Algorithm (ULA)}.
	Given a small step-size $h>0$, and a random initialization, it iterates as follows
	\begin{equation}
		\label{eq:ULA}
		\hatx_0 \sim \mu, \quad \hatx_{k+1} = \hatx_k - h\nabla V(\hatx_k) + \sqrt{2}\tond*{B_{(k+1)h} -B_{kh}},
	\end{equation}
	where $(B_t)_{t\geq 0}$ is a standard $d$-dimensional Brownian motion. 
	This is the Euler--Maruyama discretization of the \emph{Langevin dynamics (LD)}
	\begin{equation}
		\label{eq:LD}
		X_0\sim \mu, \quad dX_t = -\nabla V(X_t) dt + \sqrt{2}dB_t.
	\end{equation}
	Under mild assumptions on $V$, the continuous-time process~\eqref{eq:LD} is ergodic with respect to $\pi$ \cite{PavBook}. The discretization scheme \eqref{eq:ULA}, however, is biased, which means that its limiting stationary distribution $\hat \pi$ is in general different from $\pi$. Hence, ULA faces a basic discretization trade-off. A smaller step-size $h$ leads to a smaller error in the bias, but also increases the number of iterations required to converge to $\hat\pi$. One can remove this bias by adding a Metropolis-Hastings accept/reject correction, leading to the Metropolis-adjusted Langevin algorithm (MALA). However, in high dimension, this correction introduces step-size restrictions and dimension-dependent mixing \cite{RobertsMALA,wu2022minimax}. For this reason, many large-scale applications use ULA directly.

	\paragraph{Empirical observations}
	In many high-dimensional Bayesian problems, ULA appears to estimate statistical functionals of interest accurately even when the step size and the number of iterations are not chosen as a function of dimension, according to the classical theory \cite{Dalalyan2017,DurMou2017}. This is especially true when the functionals of interest are simple observables such as the posterior mean, or depend only on a few coordinates~\cite{welling2011bayesian,simon2025benchmarking,paulin2025sampling,chada2023unbiased,sommer2026position,PeVaZy2020,fearnhead2024scalable}.
	
	We illustrate this phenomenon in Figure \ref{fig:double_descent} for Bayesian logistic regression with simulated data and a ridge-separable potential of the form \eqref{eq:potential-stats-intro}. We plot the error $\abs*{m-\hat m}$ between the means $m,\hat m$ of $\pi$ and of the biased limit $\hat \pi$ respectively. The error is small both in an overparameterized regime, where $N\ll d$, and in an asymptotic regime where $N\gg d$, while it is larger near the proportional regime $N \asymp d$. This error plot produces a bell-shaped dependence of the ULA bias on the interplay between $N$, $d$. A similar bell-shaped phenomenon was previously observed in the Bayesian probit model~\cite{Anceschi03042023}, which was referred to as reminiscent of the \emph{double descent} phenomena in machine learning theory \cite{Hastie2022}. 
	\begin{figure}[!ht]
		\centering
		\includegraphics[width=0.328\textwidth]{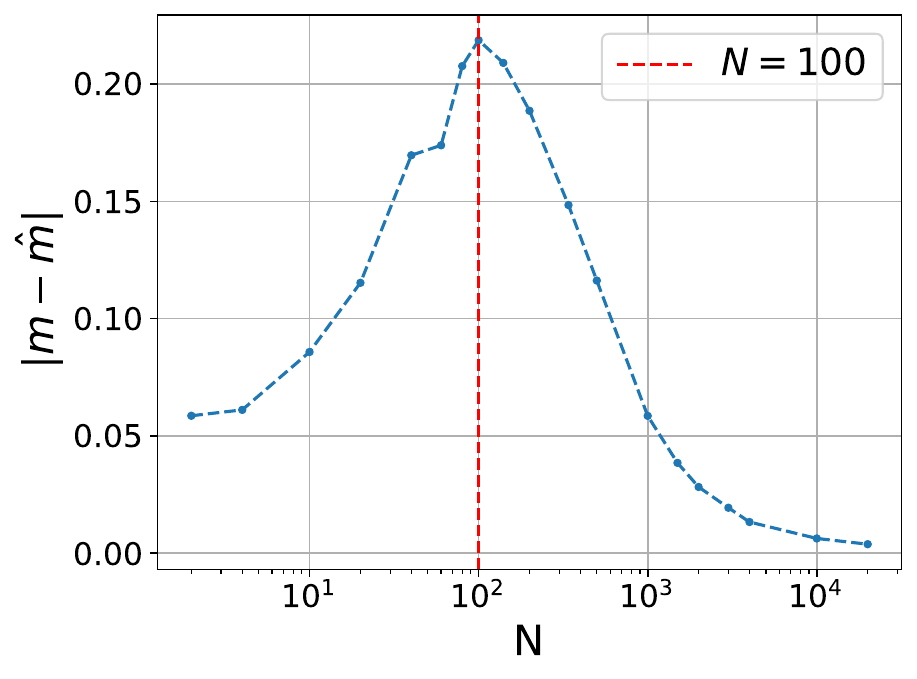} 
		\includegraphics[width=0.328\textwidth]{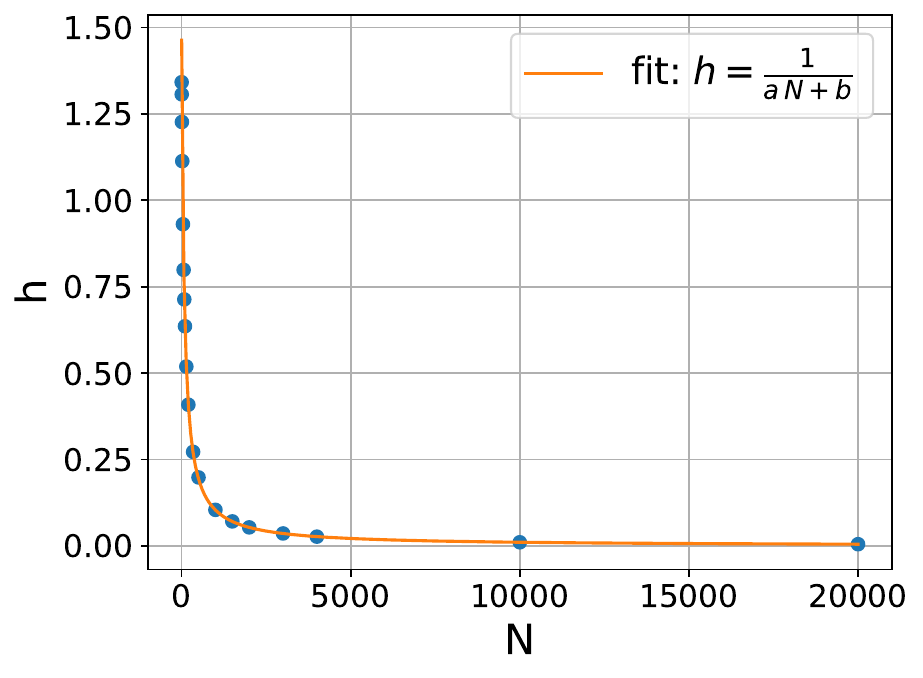}
		\includegraphics[width=0.328\textwidth]{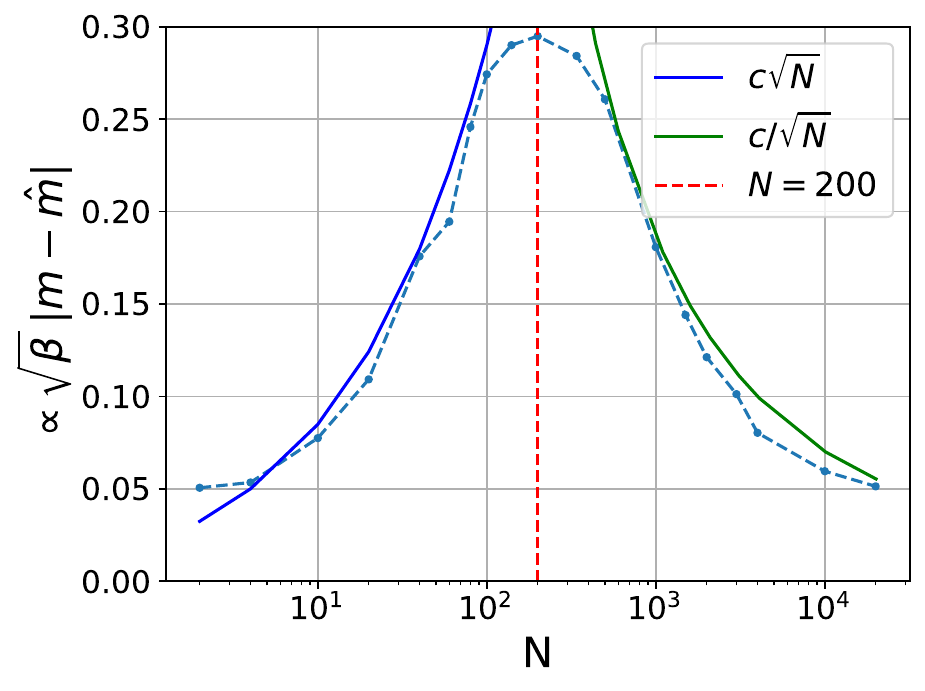}
		
		\caption{$20$-dimensional Bayesian logistic regression with a mean-zero and variance one Gaussian prior, simulated data $a_i \overset{\text{i.i.d.}}\sim \mathcal{N}(0, \frac{1}{d}I_d),$ for $i = 1,\dots,N$, with corresponding simulated labels defined via $y_i \mid a_i, w_{\mathrm{true}} \sim \mathrm{Bernoulli}\left(p_{i}\right)$ with probabilities, $p_i = \sigma(a_i^{\top} w_{\mathrm{true}}) := 1/(1 + \exp{\left(-a_i^{\top} w_{\mathrm{true}}\right)})$, where $w_{\mathrm{true}} \sim \mathcal{N}\!\left(0,\, \frac{1}{d} I_d \right)$. Left: Bell-shaped profile for bias in the mean. Middle: Stepsize scaling is in accordance with Lemma \ref{lem:high_prob_cond_scal}. Right: Bell-shaped profile for rescaled bias in the mean.  }
		\label{fig:double_descent}
	\end{figure}

	The main motivation of this paper is to theoretically justify these observations, and more generally, to understand in which situations it is possible to accurately estimate the expectation of relevant test functions via \eqref{eq:ULA} with a number of steps independent of the dimension.
	
	\paragraph{Classical guarantees and their limitation}
	The problem of quantifying the performance of the Unadjusted Langevin Algorithm \eqref{eq:ULA}, by proving error bounds in the approximation $\mu_k \coloneqq \law\tond*{\hat{X}_k} \approx \pi$ in several metrics, has given rise to an important line of work (cf. \cite{che-book-2024+} and \cite{DuEb24} for an overview).
	An important choice of metric is the $2$-Wasserstein distance, 
	defined for probability measures $\mu,\nu \in \pp_2\tond*{\R^d}$ with finite second moment by
	\begin{equation}\label{eq:w2}
		W_2(\mu,\nu) = \sqrt{\inf_{X,Y} \expe{\abs*{X-Y}^2}},
	\end{equation}
	where the infimum runs over 
	all $\R^{d}$-valued random vectors 
	$X$ and $Y$
	defined on the same probability space 
	$(\Omega, \mathcal{F}, \mathcal{P})$
	with
	$\law(X) = \mu$ and $\law(Y) = \nu$.

	By Kantorovich duality, results establishing upper bounds on $W_2(\mu_k,\pi)$ for $k$ large enough provide a possible justification for using \eqref{eq:ULA} to estimate expectations of Lipschitz test functions with respect to $\pi$.
	Classical Wasserstein guarantees for the Unadjusted Langevin Algorithm, however, require choosing a step-size scaling inversely with the dimension \cite{che-book-2024+}, similarly to what is enforced for the Metropolis correction. Therefore, they are unable to explain the improved empirical performance of \eqref{eq:ULA} often seen in practice, and in particular in the setting of \eqref{eq:potential-stats-intro}.
	To illustrate this limitation, we recall below a prototypical result of this type, which holds under the following classical \emph{log-concavity} assumption.
	\begin{assumption}\label{ass:strong-log-conc-log-smooth}
		The potential $V\in C^2\tond*{\R^d}$ satisfies 
		\[
		\alpha I_d \preccurlyeq \hess V \preccurlyeq \beta I_d,
		\]
		for some $0<\alpha <\beta$. We denote the condition number by $\kappa \coloneqq \frac{\beta}{\alpha}$. 
	\end{assumption}
	\begin{proposition}{\cite{pedrotti2026wasserstein}}
		\label{prop:w2-guarantess-ula-strong-log-conc}
		Under Assumption \ref{ass:strong-log-conc-log-smooth}, if  $h\le \frac1\beta$, we have
		\begin{equation}\label{eq:classical-w2-convergence}
			W_2\tond*{\mu_k, \pi} \le \left(1-\alpha h\right)^{k} W_2(\mu,\hat \pi) + 6h\sqrt{\kappa \beta d}.
		\end{equation}
		In particular, 
		\begin{equation}\label{eq:w2-bias-strong-log-conc}
			W_2\tond*{\pi,\hat \pi} \leq 6h\sqrt{\kappa \beta d}.
		\end{equation}
	\end{proposition}
	The error bound in \eqref{eq:classical-w2-convergence} is made of two terms: the first is exponentially decaying, and reflects the exponential convergence of \eqref{eq:ULA} to $\hat \pi$. The second corresponds to the error due to the bias $\hat \pi \neq \pi$: making it smaller than some desired accuracy requires choosing a step-size $h$ scaling with $\frac1{\sqrt{d}}$, and correspondingly the iteration complexity of the algorithm scales with the dimension too. Unfortunately, this dimension dependence typically cannot be  removed when measuring errors in the $W_2$-metric, even in the simplest case where $\pi$ is a Gaussian measure.

	\subsection{Our contributions}
	
	We propose to use the subspace robust Wasserstein distance $\mswm$ from \cite{paty2019subspace} to quantify the convergence of ULA, instead of the classical Wasserstein-2 distance. We characterize the iteration complexity to sample from ridge-separable potentials \eqref{eq:potential-stats-intro} in this metric, thus  bridging the gap between the empirical observations discussed above and the classical theoretical guarantees. More generally, we also provide bounds on the asymptotic bias of ULA for target distributions that correspond to perturbations of Gaussian measures, revealing a dimension-free behavior when the perturbation is small enough. 
	
	\paragraph{Subspace robust Wasserstein distance}
	The subspace robust Wasserstein distance $\mswm$ was introduced in \cite{paty2019subspace}, and is defined for $\mu,\nu \in \pp_2(\R^d)$ by
	\[
	\mswm^2\tond*{\mu,\nu} = \inf_{X,Y} \opnorm{\expe{(X-Y) (X-Y)^T}},
	\]
	where again the infimum is over all couplings $(X,Y)$ of $\mu,\nu$. By comparison, notice that if we replace $\opnorm{\expe{(X-Y) (X-Y)^T}}$ with $\trace{\expe{(X-Y) (X-Y)^T}}$, we recover the standard Wasserstein-2 distance $W_2^2(\mu,\nu)$: in particular, 
	we always have 
	\[
	\frac{W_2(\mu,\nu)}{\sqrt{d}}\le \msw{\mu, \nu} \le W_2(\mu,\nu).
	\]
    A first crucial observation that we make in this paper is that, when measuring the asymptotic bias of \eqref{eq:ULA}, the first inequality above is approximately achieved for some distributions, in the sense that $W_2(\pi,\hat\pi)$ scales with $\sqrt{d}$ but $\msw{\pi,\hat \pi}$ is dimension-free: this is easily checked, for example, when $\pi$ is Gaussian, cf. our Proposition \ref{pro:msw-bias-gaussian} below.

    As a second important motivation for choosing this metric, observe that upper bounds on the bias $\msw{\pi,\hat\pi}$ naturally yield guarantees for the estimation of low-dimensional observables. Indeed, for an integer $k\in \N_{>0}$ and  $\theta = \tond*{\theta_1, \ldots, \theta_k} \in \tond*{S^{d-1}}^k$, denote by $L_{\theta} \colon \R^d \to \R^k$ to denote the linear map
	\[
	x \to \tond*{\scal{x,\theta_1}, \ldots, \scal{x,\theta_k}},
	\]
	and for a measure $\mu \in \pp \tond*{\R^d}$, denote the corresponding pushforward measure by 
	$
	\mu_\theta \coloneqq L_\theta \#\mu \in \pp\tond*{\R^{k}}$.
	Then, we have the inequality 
	\[
	W_2(\mu_\theta, \nu_\theta) \le \sqrt{k} \,\msw{\mu,\nu}.
	\]
	for all $k \in \N_{>0}$ and $\theta\in \prosph{k}$. 
	Besides low-dimensional observables, the $\mswm$-metric also controls the difference between the mean of two distributions, and is intimately related to this quantity. In fact, for any coupling $(X,Y)$ of $(\mu,\nu)$ it is not difficult to check that
	\begin{equation}  \label{eq:general-bounds-mswm}
		\abs*{\int x\dd\mu -\int x \dd\nu}\le \msw{\mu,\nu} \le \sqrt{\abs*{\int x \dd\mu  - \int x \dd \nu}^2+\opnorm{\covrv{X-Y}}} .
	\end{equation}

	\paragraph{Bell-shaped mixing profile for ridge-separable potentials}
	Equipped with this metric, we 
	establish bounds on
	the iteration complexity of sampling in $\mswm$-accuracy from $\pi\propto e^{-V}$ with $V$ as in \eqref{eq:potential-stats-intro}.
	In order to have scale-invariant results, we actually consider the rescaled metric $\sqrt{\alpha}\mswm$, where $\alpha$ is the strong log-concavity parameter of $\pi$ (this is a natural and standard procedure, cf. \cite[Rmk. 4.1.4]{che-book-2024+}). 
	Our main result in this setting
	distinguishes between three different regimes, and can be informally summarized as follows.
	\begin{theorem}[Informal] \label{thm:informal-ridge}
		Consider a ridge-separable potential $V$ as in \eqref{eq:potential-stats-intro}, with i.i.d sub-Gaussian data, smooth strongly convex functions $u_i\in C^3(\R)$ and $\lambda$ of order $\expe{\abs*{a_i}^{2}}$. 
		Then, with high-probability, setting $\alpha = \sup{\left\{\mathfrak{m} >0 \mid \hess V \succeq \mathfrak{m} I_d\right\}}$, we have that
		the iteration complexity of sampling in a fixed $\sqrt{\alpha}{\mswm}$-accuracy $\eps>0$ is upper-bounded by
		\begin{enumerate}
			\item $\frac{C\sqrt{N}}{\eps}$ in the regime $N\lesssim d$,
			\item $\frac{C}{\eps}$ in the regime $N\gtrsim d^{2}$,
		\end{enumerate}
		where $C > 0$ is a dimension-free constant up to logarithmic factors.
		Furthermore, when $N\asymp d$,
		the iteration complexity can scale as  $\frac{\sqrt{d}}{\eps}$.
	\end{theorem}
	In summary, the picture that emerges is illustrated in Figure \ref{fig:theory}.
	In the \emph{overparameterized regime}, where $N  \lesssim 1$, we provide a dimension-free iteration complexity to sample in $\mswm$-accuracy with \eqref{eq:ULA}; somewhat more surprisingly, we also prove a similar phenomenon in an \emph{asymptotic regime}, where $N\gtrsim d^2$.
	On the other hand, in the proportional regime where $N\asymp d$, we consider a specific choice of $V$ 
	illustrating that dimension-free rates are not possible in this case, and provide further numerical experiments to confirm this.
	\begin{figure}[http!]
		\centering
		\includegraphics[width=0.55\textwidth]{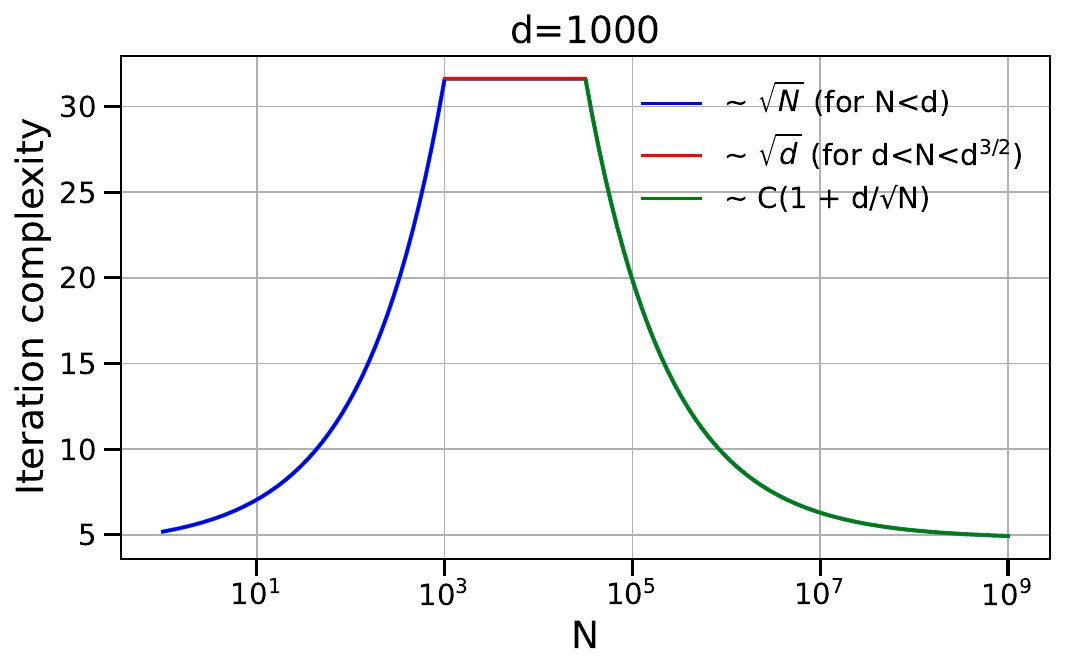}
		\caption{Illustration of Theorem \ref{thm:informal-ridge}}.
		\label{fig:theory}
	\end{figure}
	The findings of Theorem \ref{thm:informal-ridge} are consistent with the empirical observations previously discussed. Note that, under exponential convergence to a biased limit, the iteration complexity is governed by the stepsize required to control the asymptotic bias.
	Since it is difficult to measure the $\mswm$-distance numerically (between $\pi$ and $\hat\pi$), in Figure \ref{fig:double_descent} we plotted instead the error in the mean estimation, which, as suggested by \eqref{eq:general-bounds-mswm}, is an appropriate proxy.

	As anticipated, the findings of Theorem \ref{thm:informal-ridge} illustrated in Figure \ref{fig:theory} cannot be deduced from existing results in the literature.
	First, if the goal is to ensure that $\sqrt{\alpha}W_2(\mu_k,\pi) \leq \eps$ along \eqref{eq:ULA}, classical guarantees like Proposition \ref{prop:w2-guarantess-ula-strong-log-conc} require {the} number of iterations of \eqref{eq:ULA}  {to grow} polynomially with the dimension.
	In contrast, in Theorem \ref{thm:informal-ridge}, the iteration complexity to ensure $\sqrt{\alpha}\msw{\mu_k, \pi} \leq \eps$ is dimension-free (up to logarithmic factors) in the overparameterized and asymptotic regimes.
	
	Closer to our work, a related recent line of research has also focused on replacing classical distances with weaker metrics to provide dimension-free bounds on the bias \cite{che-che-wee-wea-2024,lac-zho-2025}. 
	The work \cite{che-che-wee-wea-2024} introduced in particular an $\ell^\infty$-version of Wasserstein distance, and established dimension-free bound (up to logarithmic terms) on the $\ell^\infty$-Wasserstein bias for distributions with sparse interactions.
	The sparsity assumptions considered in \cite{che-che-wee-wea-2024, lac-zho-2025}, however, are often not satisfied in statistics and machine learning applications, and in particular for the ridge-separable potentials \eqref{eq:potential-stats-intro}. Imposing an approximate product measure structure in this setting is not realistic, and it is more natural to consider assumptions that are not formulated in terms of a coordinate system. Indeed, applying an isometric transformation to the corresponding measure $\pi$ results 
	in a new potential again of the form $\eqref{eq:potential-stats-intro}$ (with new transformed data $(\tilde a_i)_i$).
	In contrast, both for the potentials \eqref{eq:potential-stats-intro} and more generally in Section \ref{sec:iso-deloc-gauss-pert} we provide upper bounds on the $\mswm$-bias under rotationally invariant assumptions on the target measure $\pi$.

	\paragraph{Subspace robust Wasserstein bias for Gaussian perturbations}
	With the aim to prove Theorem \ref{thm:informal-ridge}, we establish in Theorem \ref{thm:deloc-bias-gauss-perturb-2} upper bounds on the subspace-robust Wasserstein bias for more general distributions,
	which are of independent interest.
	More precisely, we consider distributions satisfying the classical Assumption \ref{ass:strong-log-conc-log-smooth}, but also 
	an additional third-order condition $\opnorm{\nabla^3 V} \le M$. 
	Correspondingly, we define a scale-invariant parameter $\scalenumber \coloneq \frac{M}{\alpha^{3/2}}$, to quantify closeness to a Gaussian distribution (which correspond to $M = \scalenumber = 0$). As $\scalenumber\to 0$, our bounds recover in particular the dimension-free inequality $\msw{\pi,\hat \pi} \le \sqrt{\beta}h$, valid when $\pi$ is Gaussian.

	\paragraph{Overview of proof techniques}
	
	Theorem \ref{thm:deloc-bias-gauss-perturb-2}, the Gaussian-perturbation result,  is proved by carefully analyzing the synchronous coupling of \eqref{eq:ULA} and \eqref{eq:LD}, separating the effect of the Gaussian components (which give rise to dimension-free error terms), from the perturbation part.

	The proof of Theorem \ref{thm:informal-ridge} proceeds by treating the three regimes separately. In the large-sample regime $N\gg d$, due to posterior concentration, we can treat the posterior as a Gaussian perturbation and apply our Theorem \ref{thm:deloc-bias-gauss-perturb-2}. For ridge-separable potentials with sub-Gaussian data, this requires high-probability estimates on the second and third derivatives. In the overparameterized regime $N\ll d$, the data-dependent part of the potential is a low-rank perturbation of a Gaussian, and this effective rank replaces the dimension $d$ in our bounds.
    Finally,  on the other hand, in the proportional regime $N\asymp d$ we construct an explicit example showing that dimension-free guarantees fail in general: the $\mswm$-bias, and even the bias in the mean, can grow with dimension for typical data.

	\paragraph{Paper organization}
	This paper is organized as follows. 
	In Section \ref{sec:iso-deloc-gauss-pert}, we  prove Theorem  \ref{thm:deloc-bias-gauss-perturb-2}, our general upper bounds on the $\mswm$-bias for Gaussian perturbations. Section \ref{sec:appl-stats-models} contains the proof of Theorem \ref{thm:informal-ridge} and a more in-depth discussion of the results.
	In Appendix \ref{sec:sliced-wasserstein} we further discuss the $\mswm$-distance and compare it with its variants; in particular, 
	we discuss results in expectation/high-probability about the bias in ``typical directions'', with implications for estimating the expectation $\int f \dd\pi$ for certain observables. Finally, in  Appendix \ref{sec:app:additional_experiments} we provide additional numerical experiments in the spirit of Figure \ref{fig:double_descent}, while the remaining sections of the Appendix
	contain omitted proofs and auxiliary technical results: these include the derivation of high-probability bounds needed in Section \ref{sec:appl-stats-models}, as well as several bounds on $L^p$-norms appearing in the proofs in Section \ref{sec:iso-deloc-gauss-pert}.

	\section{Subspace robust Wasserstein bias for Gaussian perturbations}
	\label{sec:iso-deloc-gauss-pert}
	In this section, we prove our main result about the $\mswm$-bias for Gaussian perturbations.  
	Our first motivation for considering Gaussian perturbations is the following proposition, of which we give a simple proof in Appendix \ref{sec:omitted-proofs}.
	It shows that, when $\pi$ is a Gaussian measure, the $\mswm$-bias is dimension-free, unlike what happens when using the standard Wasserstein-$2$ distance.
	\begin{proposition}\label{pro:msw-bias-gaussian}
		Consider a Gaussian measure $\pi \propto e^{-V}$ with $V(x) =\frac12 \scal*{x-m, A(x-m)}$ and $\alpha I_d \preccurlyeq A \preccurlyeq \beta I_d$ for some constants $0<\alpha<\beta$.
		Then, if $h\leq \frac1\beta$, we have
		\begin{equation}\label{eq:msw-bias-gaussian}
			\msw{\pi,\hat \pi} \lesssim \sqrt{\beta}h.
		\end{equation}
		However, 
		\[
		W_2(\pi, \hat \pi) \gtrsim \sqrt{\alpha d}h.
		\]
	\end{proposition}
	The second motivation comes from the 
	key insight that ridge-separable distributions are closely related to Gaussian measures, at least in the overparameterized regimes $N\ll d$, thanks to the Gaussian prior component, and in the regime $N\gg d$, where asymptotic results like the Bernstein-Von Mises Theorem play an important role \cite{VanderVaart98}.

	To go beyond the Gaussian example, our focus in this section is on distributions which are strongly log-concave and log-smooth as in Assumption \ref{ass:strong-log-conc-log-smooth}, but also satisfying the following additional third-order smoothness condition.
	\begin{assumption}
		\label{ass:third-order-smoothness}
		The potential $V\in C^3\tond*{\R^d}$ satisfies 
		\[
		\opnorm{\nabla^3 V} \le M
		\]
		for some constant $M\geq0$,
		where we define the operator norm of
		$A \in \mathbb{R}^{d \times d \times d}$ by
		\[
		\|A\|_{\textnormal{op}} \coloneqq \sup\left\{\left.\sum^{d}_{i,j,k=1}A_{ijk}x_{i}y_{j}z_{k} \right| \sum_{i}x_i^2\leq 1,\sum_{j}y_j^2\le 1, \sum_{k}z_k^2\le 1\right\}.
		\]
		Equivalently,  for all $x,y\in \R^d$
		\[
		\opnorm{\hess V(y) - \hess V(x)} \le M \abs*{x-y}
		\]
		and 
		\[
		\abs*{\nabla V(y) -\nabla V(x)-\hess V(x)(y-x)} \leq \frac{M}{2}\abs*{x-y}^2.
		\]
	\end{assumption}
    For distributions satisfying Assumptions \ref{ass:strong-log-conc-log-smooth} and \ref{ass:third-order-smoothness}, we naturally introduce the following (scale-invariant) parameter of $\pi$, in the spirit of the condition number:
		\begin{equation}
			\scalenumber \coloneqq \frac{M}{\alpha^{\frac{3}{2}}}.
		\end{equation}
	Notice that the case $M=\scalenumber=0$ corresponds to $\pi$ being a Gaussian measure, and so Assumption \ref{ass:third-order-smoothness} quantifies the closeness of $\pi$ to being Gaussian.

	Our main result for this class of distributions is the following theorem, which provides an upper bound on the subspace robust Wasserstein asymptotic bias of ULA, as well as an  improved bound for the error in the mean estimation, which goes to $0$ as $\scalenumber \to 0$. We also record a new bound on the bias in the standard $W_2$-metric of independent interest, since it follows easily with the same proof; this $W_2$-bound  is closely related to the estimate   $W_2(\pi,\hat \pi)  \lesssim h\kappa \sqrt{\beta d}+ hd\scalenumber\sqrt{\alpha}$ obtained in \cite[Theorem 5]{dal-kar-2019} under the same assumptions.
	\begin{theorem}[$\mswm$-bias of Gaussian perturbations]
    \label{thm:deloc-bias-gauss-perturb-2}
		Under Assumptions \ref{ass:strong-log-conc-log-smooth} and \ref{ass:third-order-smoothness}, if $h\le \frac{1}{\beta}$ 
		we have that
		\begin{align}
			\msw{\pi,\hat \pi} &\lesssim h \sqrt{\beta}+  h \sqrt{\beta} \scalenumber d,
			\\
			\abs*{{m}-\hat m}  & \lesssim  h \sqrt{\beta} \scalenumber d,
			\\
			W_2(\pi,\hat \pi) & \lesssim h \sqrt{\beta d}  + h \sqrt{\beta} \scalenumber d.
		\end{align}
	\end{theorem}

	\begin{remark}
		The improvement in the bound on $\msw{\pi,\hat \pi}$ compared to $W_2(\pi,\hat \pi)$ might seem minor, by removing just the dimension dependence in the first term. However, as it will be clear from the discussion in Section \ref{sec:appl-stats-models},  for our applications to ridge-separable potentials
		this will be crucial to deduce a dimension-free iteration complexity in 
		the rescaled metric $\sqrt{\alpha}\mswm$.
	\end{remark}
	
	For the proof of Theorem \ref{thm:deloc-bias-gauss-perturb-2}, we will also make use of the following standard results, which we prove in Appendix \ref{sec:app:auxiliary-ld-ula}.
	\begin{lemma} \label{lem:moment_bounds}
		Under Assumption \ref{ass:strong-log-conc-log-smooth}, let $x^*\coloneqq \argmin V$. Then we have for $X\sim \pi$ and any $u\in \R^d$
		\begin{equation}\label{eq:moment-bounds}
			\norm*{X-x^*}_{L^4}  \lesssim \sqrt{\frac{d}{\alpha}}, \qquad \norm*{\nabla V(X)}_{L^4} \lesssim \sqrt{\beta d}, \quad \textnormal{and} \quad \norm*{u \cdot \nabla V(X)}_{L^4} \lesssim \sqrt{\beta} \abs*{u},
		\end{equation}
		Moreover, define
		$A:=\mathbb{E}_\pi\nabla^2V
		=\mathbb{E}_\pi[\nabla V\nabla V^\top]$ and $ Q (x)=\nabla^2V(x)-A$.
        If Assumption \ref{ass:third-order-smoothness} holds, then 
		\begin{equation}
			\|Q(x)\|_{\textnormal{op}} \lesssim M\left(|x-x^*|+\sqrt{\frac d\alpha}\right).
			\label{eq:Q-bound}
		\end{equation}
	\end{lemma}
	Additionally, we will need the following elementary matrix bounds, proved in Appendix \ref{sec:omitted-proofs}.
	\begin{lemma} \label{lem:matrix-estimates-ula}
		Let $0 < \alpha \leq \beta$ and $A$ be a symmetric $d\times d$ matrix satisfying $\alpha I_d \preccurlyeq A\preccurlyeq \beta I_d$. If $0<h\leq 1/\beta$, then
		\begin{equation}
			h^2\sum_{j=0}^\infty\|A^{3/2}\tond*{I_d-hA}^j\|_{\textnormal{op}}
			\lesssim h\sqrt{\beta},
			\label{eq:smoothing}
		\end{equation}
		and
		\begin{align}
			h^3\sum_{j=0}^\infty \tond*{I_d-hA}^j A^2\tond*{I_d-hA}^j &\preceq h^2A.
			\label{eq:gaussian-sum}
		\end{align}
	\end{lemma}

	\begin{proof}[Proof of Theorem \ref{thm:deloc-bias-gauss-perturb-2}]

		Let $X_0=\hatx_0\sim\pi$ and consider solutions $(X_t)_{t\geq0}$ and
		$(\hatx_k)_{k\in\N}$ of \eqref{eq:LD} and \eqref{eq:ULA} with synchronous coupling.
		In particular, $X_t\sim\pi$ for every $t\geq0$. We will decompose, for every $k$,
		\begin{equation}\label{eq:error-decomp-gauss-plus-remaind}
			X_{kh}-\hatx_k=G_k+R_k,
		\end{equation}
		where, uniformly in $k$,
		\begin{align}
			\label{eq:Gk-bound}
			&\expe{G_k}=0, \qquad
			\sup_{\abs*{u}=1}\norm*{\scal*{u,G_k}}_{L^2}\lesssim h\sqrt{\beta},
			\qquad \norm*{G_k}_{L^4}\lesssim h\sqrt{\beta d},
			\\
			\label{eq:Rk-bound}
			&\norm*{R_k}_{L^2}\lesssim h\sqrt{\beta}\scalenumber d.
		\end{align}
		From this, the desired conclusion follows easily, taking $k\to \infty$.
		
		To achieve the decomposition \eqref{eq:error-decomp-gauss-plus-remaind}, let us start by writing the recursion
		\begin{equation}
			\label{eq:Zk-recursion}
			\begin{split}
				X_{(k+1)h} - \hatx_{k+1} & = X_{kh} - \hat X_k - h\tond*{\nabla V(X_{kh}) - \nabla V(\hatx_k)}  - \int_0^h \tond*{\nabla V(X_{kh+t}) -\nabla V(X_{kh})} \dd t
				\\
				& = X_{kh} - \hat X_k - h\tond*{\nabla V(X_{kh}) - \nabla V(\hatx_k)}  - \int_0^h \tond*{h-s} d\nabla V(X_{kh+s})
			\end{split}
		\end{equation}
		Next, as in  Lemma \ref{lem:moment_bounds} we define the matrix $A = \expewr{\pi}{\hess V}$ and  $Q(x) = \hess V(x) - A$, to decompose the Hessian of the potential  as the sum of a constant matrix and a remainder term, in order to highlight the linear component.
		This suggests to consider the recursion
		\begin{equation}\label{eq:Gk-recursion}
			G_{0}=0,\qquad G_{k+1} = (I_{d}-hA)G_{k} - A\int^{h}_{0}(h-s)dX_{kh+s},
		\end{equation}
		and define correspondingly $R_k = X_{kh}-\hatx_k - G_k$. It remains then to prove   the two bounds \eqref{eq:Gk-bound} and \eqref{eq:Rk-bound}.
		
		Iterating the preceding recursion and using the Langevin equation, we write
		$G_k=G_k^{\mathrm{D}}+G_k^{\mathrm{B}}$, where
		\begin{align*}
			G_k^{\mathrm{D}}
			&=\sum_{j=0}^{k-1}\tond*{I_d-hA}^jA\int_0^h(h-s)
			\nabla V\tond*{X_{(k-1-j)h+s}}\dd s,
			\\
			G_k^{\mathrm{B}}
			&=-\sqrt{2}\sum_{j=0}^{k-1}\tond*{I_d-hA}^jA
			\int_0^h(h-s)dB_{(k-1-j)h+s}.
		\end{align*}
		Both terms are centered, since $X_t\sim\pi$ and
		$\expewr{\pi}{\nabla V}=0$. For every unit vector $u$, stationarity,
		triangle inequality, and the identity
		$A=\expewr{\pi}{\nabla V\nabla V^\top}$ give
		\begin{align*}
			\norm*{\scal*{u,G_k^{\mathrm{D}}}}_{L^2}
			&\leq \frac{h^2}{2}\sum_{j=0}^{k-1}
			\norm*{\scal*{A\tond*{I_d-hA}^ju,\nabla V(X_0)}}_{L^2}
			\\
			&=\frac{h^2}{2}\sum_{j=0}^{k-1}
			\abs*{A^{3/2}\tond*{I_d-hA}^ju}
			\lesssim h\sqrt{\beta},
		\end{align*}
		where the last inequality follows from \eqref{eq:smoothing}. Moreover,
		by independence of the Brownian increments and \eqref{eq:gaussian-sum},
		\begin{equation*}
			\covrv{G_k^{\mathrm{B}}}
			=\frac{2h^3}{3}\sum_{j=0}^{k-1}\tond*{I_d-hA}^jA^2
			\tond*{I_d-hA}^j
			\preccurlyeq \frac{2h^2}{3}A
			\preccurlyeq \frac{2h^2\beta}{3}I_d,
		\end{equation*}
		which proves the directional estimate in \eqref{eq:Gk-bound}.
		
		To prove the vector estimate, let $(e_i)_{i=1}^d$ be an orthonormal eigenbasis
		of $A$, with $Ae_i=a_ie_i$. Lemma \ref{lem:moment_bounds} and a geometric sum yield
		\begin{equation*}
			\norm*{\scal*{e_i,G_k^{\mathrm{D}}}}_{L^4}
			\leq \frac{h^2a_i}{2}\sum_{j=0}^{k-1}(1-ha_i)^j
			\norm*{\scal*{e_i,\nabla V(X_0)}}_{L^4}
			\lesssim h\sqrt{\beta}.
		\end{equation*}
		It follows that
		\[
		\norm*{G_k^{\mathrm{D}}}_{L^4}
		\leq\tond*{\sum_{i=1}^d
			\norm*{\scal*{e_i,G_k^{\mathrm{D}}}}_{L^4}^2}^{1/2}
		\lesssim h\sqrt{\beta d}.
		\]
		The covariance estimate above and the standard Gaussian moment bound give the
		same estimate for $G_k^{\mathrm{B}}$, which proves \eqref{eq:Gk-bound}.
		
		We now turn to $R_k$. It\^o's formula gives
		\[
		d\nabla V(X_t)
		=\tond*{\nabla\Delta V-\hess V\nabla V}(X_t)\dd t
		+\sqrt{2}\hess V(X_t)dB_t.
		\]
		Define now the $(1-\alpha h)$-Lipschitz  map $\Phi(x)=x-h\nabla V(x)$. Splitting $\hess V=A+Q$ and subtracting
		the recursion \eqref{eq:Gk-recursion} defining $G_k$ from \eqref{eq:Zk-recursion}, we obtain
		\begin{equation}
			\label{eq:rk-recursion}
			R_{k+1}=\Phi(\hatx_k+R_k)-\Phi(\hatx_k)-h\rho_k+U_k+N_k,
		\end{equation}
		where
		\begin{align*}
			\rho_k
			&=\int_0^1Q\tond*{X_{kh}-tG_k}G_k\dd t,
			\\
			U_k
			&=\int_0^h(h-s)\quadr*{Q(X_{kh+s})\nabla V(X_{kh+s})
				-\nabla\Delta V(X_{kh+s})}\dd s,
			\\
			N_k
			&=-\sqrt{2}\int_0^h(h-s)Q(X_{kh+s})dB_{kh+s}.
		\end{align*}
		We claim now that
		\begin{align}
			\label{eq:rho_k_bound}
			\norm*{\rho_k}_{L^2}&\lesssim Mhd\sqrt{\kappa}
			\\
			\label{eq:U_k_bound}
			\norm*{U_k}_{L^2}&\lesssim Mdh^2\sqrt{\kappa}
			\\
			\label{eq:N_k_bound}
			\norm*{N_k}_{L^2}&\lesssim \frac{Mdh^{3/2}}{\sqrt\alpha}.
		\end{align}
		To see \eqref{eq:rho_k_bound} notice that by \eqref{eq:Q-bound} we have
		\begin{align*}
			\abs*{\rho_k}
			&\leq \int_0^1
			\opnorm{Q(X_{kh}-tG_k)}\abs*{G_k}\dd t
			\\
			&\lesssim M\left[
			\tond*{\abs*{X_{kh}-x^*}+\sqrt{\frac d\alpha}}
			\abs*{G_k}
			+\abs*{G_k}^2
			\right].
		\end{align*}
		Therefore, by Hölder's inequality, stationarity, and
		\eqref{eq:Gk-bound},
		\begin{align*}
			\norm*{\rho_k}_{L^2}
			&\lesssim M\left[
			\tond*{\norm*{X_{kh}-x^*}_{L^4}
				+\sqrt{\frac d\alpha}}
			\norm*{G_k}_{L^4}
			+\norm*{G_k}_{L^4}^2
			\right]
			\\
			&\lesssim M\left[
			\sqrt{\frac d\alpha}\,h\sqrt{\beta d}
			+h^2\beta d
			\right]
			\lesssim Mhd\sqrt{\kappa},
		\end{align*}
		where in the last inequality we used $h\beta\leq1$.
		To see \eqref{eq:U_k_bound}, observe first that by the triangle inequality and stationarity,
		\begin{align*}
			\norm*{U_k}_{L^2}
			&\leq \int_0^h(h-s)
			\left[
			\norm*{\opnorm{Q(X_0)}\abs*{\nabla V(X_0)}}_{L^2}
			+\norm*{\nabla\Delta V(X_0)}_{L^2}
			\right]\dd s.
		\end{align*}
		By Hölder's inequality, \eqref{eq:Q-bound}, and the moment estimates,
		\begin{align*}
			\norm*{\opnorm{Q(X_0)}\abs*{\nabla V(X_0)}}_{L^2}
			&\leq
			\norm*{\opnorm{Q(X_0)}}_{L^4}
			\norm*{\nabla V(X_0)}_{L^4}
			\\
			&\lesssim
			M\sqrt{\frac d\alpha}\sqrt{\beta d}
			=Md\sqrt{\kappa}.
		\end{align*}
		Moreover, Assumption \ref{ass:third-order-smoothness} implies that
		$\abs*{\nabla\Delta V}\leq Md$, and \eqref{eq:U_k_bound} follows.
		Finally, \eqref{eq:N_k_bound} follows from Itô's isometry and stationarity. Indeed,
		\begin{align*}
			\norm*{N_k}_{L^2}^2
			&=2\int_0^h(h-s)^2
			\expe{\fronorm{Q(X_{kh+s})}^2}\dd s
			\\
			&=2\int_0^h(h-s)^2\dd s\,
			\expe{\fronorm{Q(X_0)}^2}
			\\
			& \lesssim d h^3 \expe{\opnorm{Q(X_0)}^2}
			\\
			& \lesssim M^2 \frac{d^2}{\alpha} h^3,
		\end{align*}
		as desired.
		Finally, let $\mathcal F_k=\sigma\tond*{X_0,(B_t)_{0\leq t\leq kh}}$ and rewrite \eqref{eq:rk-recursion} as
		\begin{equation} \label{eq:Rk-recursion-rewritten}
			\begin{split}
				R_{k+1} & = \Phi(\hatx_k+R_k)-\Phi(\hatx_k)-h\rho_k+U_k+N_k
				\\
				& = \Phi(\hatx_k+R_k)-\Phi(\hatx_k) +\graf*{-h\rho_k + \expe{U_k \mid \mathcal{F}_k}} +\graf*{U_k -\expe{U_k\mid \mathcal{F}_k} + N_k }.
			\end{split}
		\end{equation}
		Since the last term conditioned on $\mathcal{F}_k$ is a mean-zero random variable, the bounds \eqref{eq:rho_k_bound}, \eqref{eq:U_k_bound}, \eqref{eq:N_k_bound} the $(1-\alpha h)$-Lipschitzness of $\Phi$ imply that that for an absolute constant $C>0$
		\begin{align*}
			\norm*{R_{k+1}}_{L^2}^2
			&\leq \tond*{(1-\alpha h)\norm*{R_k}_{L^2}
				+CMdh^2\sqrt\kappa}^2+\frac{CM^2d^2h^3}{\alpha}
			\\
			&\leq (1-\alpha h)\norm*{R_k}_{L^2}^2
			+\frac{CM^2d^2h^3\kappa}{\alpha}.
		\end{align*}
		This gives the desired inequality
		\[
		\sup_{k\geq0}\norm*{R_k}_{L^2}
		\lesssim \frac{Mhd\sqrt\kappa}{\alpha}
		=h\sqrt\beta\scalenumber d,
		\]
		thus concluding the proof.

	\end{proof}

	\section{Main results for ridge-separable models}
	\label{sec:appl-stats-models}
	As discussed in the introduction, our focus here is on a class of common models in Bayesian applications, where the potential takes the form 
	\begin{equation}
		\label{eq:potential-statistics-application}
		V(x) = \frac{\lambda}{2}\abs*{x}^2 + \sum_{i=1}^N u_i\tond*{\scal*{a_i, x}},
	\end{equation}
    for a parameter $\lambda>0$, i.i.d datapoints $(a_i)_i$ in $\R^d$, and functions $u_i\colon\R\to\R$.
    
    In this section, we establish iteration complexity bounds in $\mswm$-accuracy for sampling from $\pi$, distinguishing three regimes to complete Theorem 1: for the \emph{asymptotic regime}, where $N\gg d$, we provide Theorem \ref{thm:asymtotic-regime}; for the \emph{proportional regime}, where $N\asymp d$, we provide  Theorem \ref{thm:proportional-regime};  for the \emph{overparameterized regime}, where $N\ll d$, we provide  Theorem \ref{thm:overparameterized-regime}.

	In order to analyze and compare between these different regimes, we now place ourselves in the following specific setting.
	\begin{assumption}\label{ass:ridge-separable-assumption}
		\begin{enumerate}
			\item The functions $u_i\in C^3\tond*{\R}$ are such that  $\bar \alpha \le u_i''\le \bar \beta$ and $\abs*{u_i'''}\le \bar M$ for some constants $0\le \bar \alpha \le \bar \beta $ and $ \bar M \ge 0$.
			\item The points $(a_i)_{i=1}^N$ are i.i.d.
			from a centered sub-Gaussian distribution $\nu \in \pp\tond*{\R^d}$ such that $\Sigma \coloneqq \cova{\nu} \succ 0$. We denote by 
			\begin{equation}
				\label{eq:condition-number-covariance}
				\condcov \coloneqq \frac{\lambda_{\max}\tond*{\Sigma}}{\lambda_{\min}\tond*{\Sigma}}
			\end{equation}
			the condition number of $\Sigma$.
			
			Furthermore, $\nu$ is such that there exists a constant $K\ge 1$ for which
			\begin{align}
				&\norm*{\scal*{a_i,\theta}}_{\psi_2} \leq K \sqrt{\scal*{\theta,\Sigma \theta}} \quad \text{ for all } \theta \in S^{d-1},
			\end{align}
			where $ \norm*{\cdot}_{\psi_2}$ denotes the sub-Gaussian norm of a random variable \cite{ver-2026}.
		\end{enumerate}    
	\end{assumption}
	
	Notice that Assumption \ref{ass:ridge-separable-assumption} implies in particular that Assumption \ref{ass:strong-log-conc-log-smooth} is satisfied too for some $0<\alpha\le\beta$, since 
	\[
	\lambda I_d + \bar\alpha\sum_{i=1}^N a_i a_i^\top
	\preccurlyeq \hess V(x)
	\preccurlyeq \lambda I_d + \bar\beta\sum_{i=1}^N a_i a_i^\top.
	\]

	\subsection{Asymptotic regime: $N\gg d$}
	
	The goal of this first subsection is to show that, when $N\gtrsim d^{\frac{3}{2}}$, our upper bound on the iteration complexity of sampling from $\pi$ in subspace robust Wasserstein accuracy decreases with $N$, and in particular becomes dimension-free in the regime $N\gtrsim d^2$ (up to log factors). This is consistent with the empirical observations of Figure \ref{fig:double_descent}, where the rescaled bias in the means decreases in the large-sample regime.
	Our main result in this regime is the following theorem.

	\begin{theorem} \label{thm:asymtotic-regime}
		Let $\pi \propto e^{-V}$ for a ridge-separable potential $V$ as in \eqref{eq:potential-statistics-application}, and suppose that Assumption \ref{ass:ridge-separable-assumption} holds with $\bar \alpha >0$.
		Then, there  exists an absolute constant $C>0$
		such that, if  
		\[
		N \ge C \max\left\{K^4 d,\,d^{3/2}\right\},
		\] 
		with probability at least $1-4e^{-d}$ over the datapoints $(a_i)_i$ the following holds.
		Denoting by $\alpha = \sup{\left\{\mathfrak{m} >0 \mid \hess V \succeq \mathfrak{m} I_d\right\}}$ the log-concavity constant of $\pi$,
		for every $0<\eps<1$, there exists a  step-size $h>0$ such that along \eqref{eq:ULA} we have
		\[
		\sqrt{\alpha}\msw{\mu_k,  \pi} \le \eps
		\]
		provided that 
		\[
		k \ge  \frac{\tilde C(K, \condcov, \bar\alpha,  \bar \beta, \bar M)\tond*{1 + \frac{d}{\sqrt{N}}}}{\eps} \log\tond*{1+ \frac{\sqrt{\alpha} W_2(\mu,\hat \pi)}{\eps}},
		\]
		where $\tilde C(K, \condcov, \bar\alpha,  \bar \beta, \bar M)>0$ is a positive constant depending only on the specified parameters.
	\end{theorem}
    \begin{remark}
		The assumption $\bar\alpha > 0$ is not satisfied in Bayesian logistic regression or in many other generalized linear models (including in particular the potential in Figure~\ref{fig:double_descent}). For such models, \cite{nickl2022polynomial,chak2025complexity,altmeyer2022polynomial} instead use local convexity arguments to derive complexity guarantees in the standard $W_2$-distance, with dimension-dependent constants. Roughly speaking, they show with high probability that the posterior is strongly convex on a bounded region containing most of its mass, with a local strong convexity constant of order $N$ arising from the summed Hessian.
		Although our assumption does not hold for these statistical potentials, it still provides a useful conceptual explanation for the behavior observed in practice, in particular the apparent absence of dimension dependence and the emergence of the Bell-shaped mixing profile illustrated in Figure \ref{fig:double_descent}.
	\end{remark}
	
	A key insight in the proof of Theorem \ref{thm:asymtotic-regime} is that, in the limit as $N\to \infty$, the ridge-separable potential becomes approximately quadratic on its natural scale. In this sense, the distribution becomes increasingly Gaussian, and we know that the $\mswm$-bias is dimension-free for Gaussian measures (cf. Proposition \ref{pro:msw-bias-gaussian}).
	We capture this intuition through the parameter $\scalenumber$, defined in Section \ref{sec:iso-deloc-gauss-pert}
	as 
	\[
	\scalenumber = \frac{\sup_x \opnorm{\nabla^3 V(x)}}{\alpha^{\frac{3}{2}}}.
	\]
	This parameter is equal to $0$ if and only if $V$ is quadratic, i.e. if $\pi$ is Gaussian. Furthermore, $\scalenumber$ is a scale-invariant parameter, meaning that 
	if we replace $\pi$ by a rescaled version $T_\#\pi$ for a linear map $T(x) = c x$, with $c\in \R\setminus\{0\}$, then $\scalenumber$ does not change.
	For ridge-separable potentials \eqref{eq:potential-statistics-application},  under Assumption \ref{ass:ridge-separable-assumption}, one could expect that for fixed $d$ both $\alpha, \beta$ and $M = \sup_x \opnorm{\nabla^3 V}$ grow linearly with $N$, with high-probability. Under this scaling, the parameter $\scalenumber$ decays to $0$ as $\sqrt{\frac1N}$ if $N\to \infty$, and at the same time the condition number $\kappa$ of $V$ remains bounded. We make this intuition rigorous in Lemma \ref{lem:high_prob_cond_scal} below, whose proof is given in Appendix \ref{sec:app:random-tensors-high-dimensions}. This allows us to interpret the target measure $\pi$ in the asymptotic regime as a third-order perturbation of a Gaussian distribution, and to apply our Theorem \ref{thm:deloc-bias-gauss-perturb-2} to provide a bound on the bias $\msw{\pi,\hat\pi}$, from which, in turn, we can control the iteration complexity to sample to a desired $\mswm$-accuracy.

	\begin{lemma}\label{lem:high_prob_cond_scal}
		Suppose that Assumption \ref{ass:ridge-separable-assumption} holds with $\bar \alpha >0$
		and consider the probability measure $\pi \propto e^{-V}$ with  potential \eqref{eq:potential-statistics-application}. 
		Then, there exists an absolute constant $C>0$ such that if $N>C\max\graf*{ K^{4}d, d^{3/2}}$ then with probability at least $1-4e^{-d}$
        \begin{align} \label{eq:cond-numb-bound-asym}
			\kappa &\le C \frac{\condcov\bar\beta}{\bar\alpha},
			\\ \label{eq:scale-numb-bound-asym}
			\scalenumber & \le C \frac{\bar M \condcov^{\frac{3}{2}}K^3}{\bar \alpha ^{\frac{3}{2}}} \frac{1}{\sqrt{N}}.
		\end{align}
	\end{lemma}
	\begin{proof}[Proof of Theorem \ref{thm:asymtotic-regime}]
		Fix $\mu\in\pp_2\tond*{\R^d}$ and $0<\eps<1$ and consider a step-size $h\le \frac1\beta$, where 
		$\beta =\inf{\left\{\mathfrak{L} >0 \mid \hess V \preceq \mathfrak{L} I_d\right\}}$ is the log-smoothness of $\pi$.
		Along \eqref{eq:ULA},
		\begin{equation} \label{eq:convergence-bound-mswm-asymptotic}
			\begin{split}
				\sqrt{\alpha}\msw{\mu_k, \pi} & \leq \sqrt{\alpha}\msw{\mu_k,\hat\pi} +\sqrt{\alpha}\msw{\hat\pi,\pi}
				\\
				&\leq \sqrt{\alpha} W_2\tond*{\mu_k,\hat\pi} + \sqrt{\alpha}\msw{\hat\pi, \pi}
				\\
				& \leq \sqrt{\alpha}e^{-\alpha kh} W_2\tond*{\mu,\hat\pi}+\sqrt{\alpha}\msw{\hat\pi, \pi}.
			\end{split}
		\end{equation}
		Since the first term decays exponentially fast, the iteration complexity is essentially determined by the choice of the step-size $h>0$, as imposed by the second error term.
		In what follows, we denote by $C$ an absolute constant and by $\hat C(K, \condcov, \bar\alpha, \bar M),\tilde C(K, \condcov, \bar\alpha,  \bar \beta, \bar M)$ positive constants depending only on the specified parameters.
		From Lemma \ref{lem:high_prob_cond_scal}, we know that if $ N \ge C \max\graf*{K^4 d,\,d^{3/2}}$ then with probability at least $1-4e^{-d}$ the potential $V$ satisfies Assumptions \ref{ass:strong-log-conc-log-smooth} and \ref{ass:third-order-smoothness} with $\kappa,\scalenumber$  bounded as in \eqref{eq:cond-numb-bound-asym}, \eqref{eq:scale-numb-bound-asym}.
		On this event, it follows from  Theorem \ref{thm:deloc-bias-gauss-perturb-2} that for $h\le \frac{1}{\beta}$
		\[
		\sqrt{\alpha}\msw{\pi,\hat \pi} \lesssim \hat C(K, \condcov, \bar\alpha, \bar M)\cdot \tond*{h\beta} \cdot \tond*{1 + \frac{d}{\sqrt{N}}}.
		\]
		We can choose $h=  \frac{\eps}{\beta \hat C(K, \condcov, \bar\alpha, \bar M)\tond*{1 + \frac{d}{\sqrt{N}}} }$ so that the second term in \eqref{eq:convergence-bound-mswm-asymptotic} is smaller than $\frac{\eps}{2}$.
		For such step-size, it suffices to take 
		\[
		k \ge \frac{\tilde C(K, \condcov, \bar\alpha,\bar\beta \bar M)\tond*{1 + \frac{d}{\sqrt{N}}}}{\eps} \log\tond*{1+\frac{\sqrt{\alpha} W_2(\mu,\hat\pi)}{\eps}}
		\]
		iterations of ULA so that the first term in \eqref{eq:convergence-bound-mswm-asymptotic} is less than $\eps/2$ too, as desired.
	\end{proof}

	\subsection{Proportional regime: $N \asymp d$}

	In this section, we study the proportional regime where $N \asymp d$, and we argue that in general we cannot prove dimension-free convergence rates in the rescaled $\mswm$-metric.
	In order to do this, we construct explicitly a potential $V$ of the form \eqref{eq:potential-statistics-application}, choosing $N = d$.

	\begin{theorem} \label{thm:proportional-regime} 
		For all $N = d$ large enough, there exists a potential of the form \eqref{eq:potential-statistics-application} with $\lambda = \frac14$, $\nu \in \pp \tond*{S^{d-1}}$ and $u_i \colon \R\to \R$ satisfying Assumption \ref{ass:ridge-separable-assumption}, such that the following holds for an absolute constant $C>0$.
		With probability at least $C$ over the datapoints $a_i$, for any  step-size $0<h<C$ small enough so that a stationary distribution $\hat \pi$ for \eqref{eq:ULA} exists, we have
		\begin{equation} \label{eq:mswm-bias-low-bound-prop-regime}
			\msw{\pi,\hat \pi} \ge C \sqrt{d}h. 
		\end{equation}
		Furthermore, with probability at least $C$, there exists an initial distribution $\mu\in \pp \tond*{\R^d}$ satisfying $W_2(\mu, \pi) \le 3 $ and such that
		along \eqref{eq:ULA} with $0<h<C$
		\begin{equation}\label{eq:iteration-low-bound-prop-regime}
			\msw{\mu_k,\pi} \ge (1-h/4)^k.
		\end{equation}
	\end{theorem}
	\begin{remark}
		The bounds in the above theorem are stated for the $\mswm$-metric rather than the rescaled version: this is without loss of generality, since the strong log-concavity parameter is of order $1$.
		To appreciate the result, suppose that the goal is to sample with accuracy $\msw{\mu_k, \pi}\le \eps$ for a fixed  $\eps>0$. Then, \eqref{eq:mswm-bias-low-bound-prop-regime} forces to choose a step-size $h\lesssim \frac{\eps}{\sqrt{d}}$. Consequently, from \eqref{eq:iteration-low-bound-prop-regime}, we see that this requires at least 
		\[
		k \gtrsim \frac{\sqrt{d}}{\eps}
		\]
		iterations of ULA, proving that indeed (almost) dimension-free rates are not possible in this regime.
		Notice that for the specific distribution $\nu$ constructed in Theorem \ref{thm:proportional-regime} the constant  $K$ from Assumption \ref{ass:ridge-separable-assumption} is not of order $1$, but rather $K\asymp \sqrt{\frac{d}{\log d}}$, unlike what happens for Gaussian distributions. The results of the previous subsection still imply that the iteration complexity is asymptotically dimension-free for $N$ large enough, and so the double descent phenomenon holds for this distribution, 
		although in this specific case we require a larger choice of $N$ compared to the $N\gtrsim d^2$ derived before. 
		Theorem \ref{thm:proportional-regime} provides a rigorous justification that dimension-free guarantees are in general not possible in the regime $N\asymp d$.  We complement this theoretical result with a simulation for a potential analogous to the explicit one of Theorem \ref{thm:proportional-regime}, but with Gaussian data $(a_i)_i$ satisfying Assumption \ref{ass:ridge-separable-assumption} with constant (dimension-free) $K$. The results are illustrated in Figure \ref{fig:proportional_regime_simulation}, showing that even in this setting when $N\asymp d$ the error in the mean (hence the $\mswm$-bias) scales with $\sqrt{d}$. Notice that in this setting the strong log-concavity parameter of $\pi \propto e^{-V}$ is of order $1$, as discussed before, hence without loss of generality we can plot $\abs*{m-\hat m}$ rather than $\sqrt{\alpha}\abs*{m-\hat m}$.

		\begin{figure}[http!]
			\centering
			\includegraphics[width=0.55\textwidth]{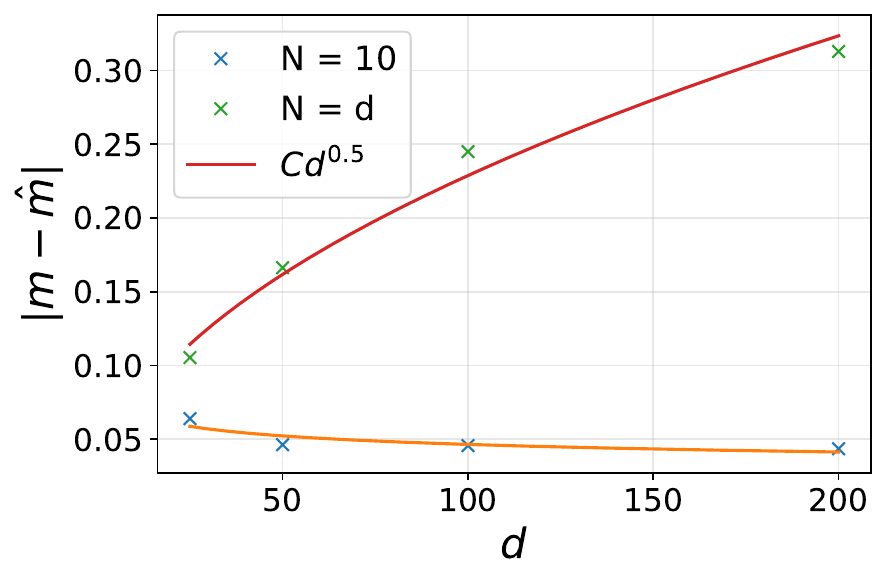}
			\caption{
				Bias of ULA as a function of the dimension $d$ for two choices of the number of ``data points'': $N=d$ and $N=10$ fixed. For each $d \in \{25,50,100,200\}$, we study the target with potential
				$
				V(x)=\frac12\|x\|^2+\sum_{i=1}^{N}\tond*{0.7 \sin\tond*{\left\langle x, a_i\right\rangle}+\abs*{\left\langle x, a_i\right\rangle}^2}$, where
				$a_i \sim \mathcal{N}(0,I_d/d),
				$
				using one sampled matrix $A \in \mathbb{R}^{N\times d}$ for each value of $d$. The reference mean $m$ is approximated by long MALA simulations with a tuned step size targeting acceptance $0.574$, and $\hat m$ is the ULA estimate obtained with step size $h = 0.2$. The figure reports the error $\abs*{m-\hat m_h}$, averaged over independent chains, for the same set of dimensions in both regimes.
			}
			\label{fig:proportional_regime_simulation}
		\end{figure}
	\end{remark}

	For the proof of Theorem \ref{thm:proportional-regime}, we will use the following lemma, whose proof is given in Appendix \ref{sec:omitted-proofs}: it constructs a centered one-dimensional distribution $\rho$, for which the biased limit of ULA has non-zero mean. The key idea for the proof of Theorem \ref{thm:proportional-regime} is then to choose the data distribution $\nu$ to be supported on the set $\graf*{\pm e_i}$ of coordinate vectors, so that the the resulting ridge-separable distribution $\pi \propto e^{-V}$is a product measure; more precisely, we ensure that with high probability there is a subset of order $d$ coordinates such that the marginal of $\pi$ on those coordinates is a product of independent copies of $\rho$. From this, it follows immediately that the the error in the mean (and to the bias $\msw{\pi,\hat \pi}$) scales with $\sqrt{d}$, as desired.

	\begin{lemma}\label{lem:one-dimensional-potential}
		There exists a one-dimensional probability measure $\rho = e^{-U} \in \pp\tond*{\R}$ with
		potential $U\colon \R\to \R$ of the form 
		\[
		U(x) = \frac{x^2}{2}+C_1\sin x +C_2 x+C_3
		\]
		for constants $C_1,C_2,C_3\in \R$ such that $\rho$ has mean $0$ and $\frac12 \le U'' \le 2$.
		Furthermore, there exists an absolute constant $C>0$ such that 
		for any step-size $0<h<C$, denoting by $\hat \rho_h$ the stationary distribution of ULA, we have that
		\[
		\abs*{\int_\R x \dd \rho  - \int_\R x \dd \hat \rho_h} = \abs*{\int_\R x \dd \hat \rho_h}\ge Ch.
		\]
	\end{lemma}

	\begin{proof}[Proof of Theorem \ref{thm:proportional-regime}]
		Throughout this proof, $C>0$ is a small absolute constant whose value is possibly adjusted (decreased) from line to line.
		Recall that we denote by $(e_i)_i$  the standard basis of $\R^d$, and 
		consider  $\nu\in \pp \tond*{S^{d-1}}$ to be the uniform distribution on the set $\graf*{\pm e_i\mid i \in [d]}$, which has cardinality $2d$; notice that $\nu$ is centered and $\cova{\nu}=\frac1d I_d$.
		Set $N= d$ and let $(a_i)_{i=1}^N = (a_i)_{i=1}^d \iid \nu$. We consider $\pi \propto e^{-{V}}\in \pp\tond*{\R^d}$ with potential $V$ defined as in \eqref{eq:potential-statistics-application}, where we also choose $\lambda = \frac14$ and $u_i (x) \coloneqq U(x) -\frac18 x^2$, where $U$ is the potential from Lemma \ref{lem:one-dimensional-potential}.
		Thus, Assumption \ref{ass:ridge-separable-assumption} is satisfied with $\lambda, \bar \alpha, \bar \beta, \bar M,\kappa_\Sigma$ all of order $1$; furthermore, the choice of $\lambda$ satisfies $\lambda \asymp \trace{\Sigma}$, which is the natural order of magnitude as we explain in the next subsection.
		By construction, $\pi$ is a $\frac{1}{4}$-log-concave product measure, i.e. we have 
		\begin{align}
			V(x) &= \sum_{i=1}^d V_i(x_i)
			\\
			V_i(x_i) & =  \frac18(1-k_+(i)-k_-(i)) x_i^2+ k_+(i)U(x_i) +k_-(i)U(-x_i)  
		\end{align}
		where $k_+(i)$ and $k_-(i)$ are defined as the number of $j\in [N]$ such that $a_j = e_i$ or $a_j = -e_i$ respectively.
		We denote now by $u\subset [d]$ the set of coordinates $i$ for which $k_+(i)=1$ and $k_-(i)=0$, and by $\pi_u \in \pp \tond*{\R^{|u|}}$ the marginal of $\pi$ on these coordinates.
		Notice that $\pi_u = \rho^{\otimes \abs*{u}}$, i.e. $\pi_u$ is a product measure corresponding to $\abs*{u}$ copies of $\rho = e^{-U}$, with $\rho$ from Lemma \ref{lem:one-dimensional-potential}.
		Furthermore, denoting by $\mathcal{A}$ the event that $\abs*{u}\ge \frac{d}{10}$ and that there exists a coordinate $i$ for which $k_+(i) = k_-(i) =0$,
		it is not difficult to check that for large $N=d \gg 1$
		we have
		\[
		\pro{\mathcal{A}} \geq C
		\]
		(where we recall that $C>0$ is a small enough absolute constant).
		It follows that, for any step-size $0<h<C$, with probability at least $C>0$
		\begin{align*}
			\msw{\pi,\hat\pi } & \ge
			\abs*{m-\hat m} 
			\\
			&\ge \abs*{\int x \dd \pi_u - \int x \dd \hat \pi_u}
			\\
			&\gtrsim  \sqrt{\abs*{u}} h
			\\
			& \gtrsim \sqrt{d}h,
		\end{align*}
		which proves the first desired bound \eqref{eq:mswm-bias-low-bound-prop-regime}.
		From this, it is not difficult to conclude that in the regime $N\asymp d$ it is not possible to establish dimension-free complexity bounds for sampling in $\mswm$-accuracy, essentially because the previous bound forces to choose a step-size scaling with $\frac1{\sqrt d}$.
		More precisely, we argue as follows to prove \eqref{eq:iteration-low-bound-prop-regime}. Recall that, given the event $\mathcal{A}$, we denote by $i$  a coordinate for which 
		$k_+(i) = k_-(i) =0$.
		Consider as initialization for \eqref{eq:ULA} the product measure $\mu$
		whose marginal on coordinate $i$ is a Dirac $\delta_1$ centered at $1$, and whose marginal on the remaining coordinates $[d]\backslash \graf{i}$ is equal to that of $\pi$. Notice that 
		\begin{align*}
			W_2(\mu,\pi) = & \sqrt{\int_{\R}  \abs*{x_i-1}^2 \dd\pi(x)} \le 3
		\end{align*}
		since $\pi_i$ is a centered  one-dimensional Gaussian with variance $4$.
		Consider now running \eqref{eq:ULA} with a step-size $0<h<C$, and denote by $R_k$ the mean of the $i$-th marginal at iteration $k$. Clearly, $R_0 = 1$, and is not difficult to check that by taking expectations in the recursion defining ULA that
		\[
		R_k = (1-h/4)^k.
		\]
		We deduce that
		\begin{align*}
			\msw{\mu_k, \pi} &\ge \abs*{R_k-\int x_i \dd \pi}
			\\
			& = (1-h/4)^k,
		\end{align*}
		which gives the desired conclusion in \eqref{eq:iteration-low-bound-prop-regime}.
	\end{proof}

	\subsection{Overparameterized regime: $N\ll d$}
	
	We first prove in Theorem \ref{thm:msw-bias-low-rank-pert} a general bound about the $\mswm$-bias for low-rank perturbations 
	of standard Gaussians. In the proof, we make use of the following tensorization bound for the $\mswm$-metric, proved in Appendix \ref{sec:omitted-proofs}.
	\begin{lemma}[Tensorization bound for $\mswm$]
		\label{lem:tensorization-mswm}
		Consider product measures $\mu = \mu_1\otimes \mu_2, \nu = \nu_1\otimes \nu_2 \in \pp_2\tond*{\R^{d_1+d_2}}$.
		We have
		\begin{equation}
			\mswm^2\tond*{\mu,\nu} \leq \mswm^2\tond*{\mu_1, \nu_1} +\mswm^2\tond*{\mu_2,\nu_2}.
		\end{equation}
	\end{lemma}
	
	\begin{theorem}\label{thm:msw-bias-low-rank-pert}
		Suppose that $\pi \propto e^{-V}$ with
		\[
		V(x) = \frac{\lambda}{2}\abs*{x}^2+ \ell\tond*{Rx},
		\]
		where $\ell\colon \R^d\to \R$ satisfies $0_d\preccurlyeq \hess \ell \preccurlyeq L I_d$, $\lambda >0$,
		and $R\in \R^{d\times d}$ is an orthogonal projection matrix with rank $r\ge 1$ such that $R^2 = R = R^T$.
		Then, if $0< h\le \frac{1}{\lambda+L}$ we have
		\[
		\msw{\pi,\hat \pi} \lesssim h\sqrt{r}\frac{\lambda+L}{\sqrt{\lambda}}.
		\]
	\end{theorem}
	
	\begin{proof} 
		Without loss of generality, after applying an orthogonal transformation, we can assume that
		\[
		V(x) = \frac{\lambda}{2}\abs*{x}^2 + \tilde \ell\tond*{(x_i)_{i=1}^r}.
		\]
		with $\tilde \ell\colon \R^r \to \R$ still satisfying $0_r \preceq \hess \tilde \ell \preceq L I_r $
         (cf. Appendix  \ref{sec:app:auxiliary-ld-ula} for additional discussion about rotational invariance of (Unadjusted) Langevin).

		Then, denoting by $\pi_1$ the marginal of $\pi$ on the first $r$ coordinates and by $\pi_2$ the marginal on the remaining ones we have 
		\begin{itemize}
			\item $\pi = \pi_1 \otimes \pi_2$ and $\hat \pi = \hat \pi_1 \otimes \hat \pi_2$;
			\item $\msw{\pi_1, \hat \pi_1} \lesssim W_2(\pi_1,\hat \pi_1) \lesssim h\sqrt{r}\frac{\lambda+L}{\sqrt{\lambda}}$ by Proposition \ref{prop:w2-guarantess-ula-strong-log-conc};
			\item $\msw{\pi_2, \hat \pi_2} \lesssim h\sqrt{\lambda}$ since $\pi_2$ is Gaussian, cf. Proposition \ref{pro:msw-bias-gaussian}.
		\end{itemize}
		The proof of the theorem is then completed by Lemma \ref{lem:tensorization-mswm}.
	\end{proof}

	We are finally ready to state our main result for ridge-separable potentials \eqref{eq:potential-statistics-application} in the overparameterized regime $N\ll d$, showing that the iteration complexity of achieving a target $\mswm$-accuracy scales with $\sqrt{N}$ rather than $\sqrt{d}$.
	\begin{theorem} \label{thm:overparameterized-regime}
		Let $\pi \propto e^{-V}$ for a ridge-separable potential $V$ as in \eqref{eq:potential-statistics-application}, and suppose that Assumption \ref{ass:ridge-separable-assumption} holds, that $N \le d-1$, and that
		\[
		\lambda \asymp\trace{\Sigma} = \expe{\abs*{a}^2}.
		\]
		Then, with probability at least $1-2e^{-d}$ over the datapoints $(a_i)_i$, for any $0<\eps<1$, there exists a sufficiently small step-size $h>0$ such that  along \eqref{eq:ULA} we have
		\[
		\sqrt{\lambda}\msw{\mu_k,  \pi} \le \eps
		\]
		provided that 
		\[
		k \ge  \frac{C\sqrt{N}\tond*{1+\bar\beta K^{2}\condcov }}{\eps}\log\tond*{1+\frac{\sqrt{\lambda}W_2(\mu,\hat \pi)}{\eps}},
		\]
		where $C>0$ is a positive absolute constant.
	\end{theorem}
	\begin{remark}
		In the setting of Theorem \ref{thm:overparameterized-regime}, observe that 
		\[
		\sup{\left\{\mathfrak{m} >0 \mid \hess V \succeq \mathfrak{m} I_d\right\}} = \lambda >0
		\]
		since  $\hess \tond*{V(x)-\lambda/2\abs*{x}^2}$ does not have full rank  for $N\le d-1$.
		Unlike in Theorem \ref{thm:asymtotic-regime}, here it is important that $\lambda>0$ is strictly positive, but $\bar \alpha$ from Assumption \ref{ass:ridge-separable-assumption} is allowed to be equal to $0$. The assumption that 
		$\lambda \asymp \trace{\Sigma}$ reflects the natural order of magnitude for this parameter, so that, for example, $\lambda$ is of order $1$ if the datapoints $a_i$ have unit norm.
	\end{remark}
	
	The result is deduced as a consequence of Theorem \ref{thm:msw-bias-low-rank-pert} (with $r \leq N$), and for this we need the following auxiliary lemma, proved in Appendix \ref{sec:app:random-tensors-high-dimensions}, to establish bounds in high-probability for the Hessian of the low-rank perturbative term.
	\begin{lemma}\label{lem:tail-prob-bound-opnom-sample-cov-2}
		Under Assumption \ref{ass:ridge-separable-assumption}, there exists an absolute constant $C>0$ such that, for all $u\ge 0$, 
		if $N\le d+u$ then with probability at least $1-2e^{-u}$ we have that
		\begin{align}
			\label{eq:upp-bound-tail-prob-norm-sample-cov-bis}
			{{\sum_{i=1}^N a_i a_i^T} \preccurlyeq N\opnorm{\Sigma}\tond*{1+CK^2{\frac{d+u}{N}}}I_d}.
		\end{align}
	\end{lemma}

	\begin{proof}[Proof of Theorem \ref{thm:overparameterized-regime}]
		Notice first that the assumptions of Theorem \ref{thm:msw-bias-low-rank-pert} are satisfied by $\pi$, with $r\le N$. 
		From Lemma \ref{lem:tail-prob-bound-opnom-sample-cov-2} with $u=d$, with probability $1-2e^{-d}$ we have the upper bound
		\[
		L \lesssim \bar\beta K^2 \opnorm{\Sigma}{d} \leq \bar\beta K^{2} \condcov \trace{\Sigma}.
		\]
		Let us then assume that this event holds.
        It  follows from Theorem \ref{thm:msw-bias-low-rank-pert} that 
		\[
		\msw{\pi,\hat \pi} \lesssim \frac{\lambda + \bar\beta K^{2} \condcov \trace{\Sigma}}{\sqrt{\lambda}}\sqrt{N}h.
		\]
		Using that
		$
		\lambda \asymp \trace{\Sigma} = \expe{\abs*{a}^2},
		$ 
		we deduce that 
		\begin{equation} \label{eq:bias-bound-overparameterized}
			\msw{\pi,\hat \pi} \lesssim h\sqrt{N\trace{\Sigma}}\tond*{1+\bar\beta K^{2}\condcov }.
		\end{equation}
		To conclude the proof, we can derive an upper bound on the iteration complexity to ensure 
		\[
		\sqrt{\lambda} \msw{\mu_k,\pi} \lesssim \eps
		\]
		for a fixed $0<\eps<1$. Arguing similarly to the proof of Theorem \ref{thm:asymtotic-regime}, from  \eqref{eq:bias-bound-overparameterized} we choose a step-size 
		$$h\lesssim \frac{\eps}{ \lambda\sqrt{N}\tond*{1+\bar\beta K^{2}\condcov }},$$ and from the exponential decay of the Wasserstein distance along \eqref{eq:ULA} we conclude that it suffices to take 
		\[
		k \gtrsim \frac{\sqrt{N}\tond*{1+\bar\beta K^{2}\condcov }}{\eps}\log\tond*{\frac{\sqrt{\lambda}W_2(\mu,\hat \pi)}{\eps}}
		\]
		iterations of the Unadjusted Langevin Algorithm.
	\end{proof}

	\paragraph{Acknowledgments} YC thanks David Dunson for fruitful discussions on the limitations of existing worst-case complexity analysis of Markov chain sampling algorithms. YC and FP acknowledge support from the European Research Council (ERC), through StG ``SamplingBeyondCS'' grant ID 101303421. PAW is partially supported by SNSF Project funding 10009286. 
	\paragraph{AI Usage Statement}
	We wrote the paper and originally derived the proofs without the use of AI. In the later stages of the project, we used ChatGPT only to improve the dependence on $\kappa$ and $\mathfrak{S}$ in Theorem \ref{thm:deloc-bias-gauss-perturb-2} (which does not affect the main conclusions of the paper and the application to Theorem \ref{thm:informal-ridge}). At this stage, we also used it for complementary proof checking. The authors retain full responsibility for all aspects of this article and certify that, to the best of their knowledge, all theoretical derivations and simulations are correct.

	\printbibliography
	
	\newpage

	\appendix
	
	\section{Additional simulation details}
    \label{sec:app:additional_experiments}
    
    In Figure \ref{fig:double_descent}, we average the results over $20$ independent datasets for each value of $N$. We estimate the posterior mean using $10^6$ samples from MALA and compare it with the mean estimated from $5\times 10^5$ samples from ULA. The ULA step-size is set to $d^{1/3}h_{\mathrm{MALA}}$, where $h_{\mathrm{MALA}}$ is tuned to target the optimal acceptance rate for MALA \cite{RobertsMALA} and estimates $c/\beta$ for some $c>0$. This choice provides a consistent step-size tuning procedure across values of $N$.
    
    Figure \ref{fig:double_descent} uses simulated data, but the same bell-shaped behavior is also observed for Bayesian logistic regression on the UCI SUSY dataset \cite{susy_279} in Figure \ref{fig:double_descent_SUSY}. This dataset consists of a classification problem from particle physics with $18$ covariates (dimensions). Its large number of observations allows us to explore the dependence of the error on $N$ across the different regimes.
    We use the first $50,000$ observations and, for each $N$, average results over $10$ subsets drawn uniformly without replacement after a single global shuffle. We precondition the covariates using a fixed invertible linear transformation, chosen so that the transformed coordinate variances are of order $1/d$.
    
    We estimate the posterior mean using two independent MALA chains of length $10^6$, after $2\times 10^5$ tuning iterations targeting the optimal acceptance rate for MALA \cite{RobertsMALA}, and compare this estimate with the mean estimated from a $10^6$-step ULA chain with step-size $d^{1/3}h_{\mathrm{MALA}}$.

    \begin{figure}[h!]
        \centering
        \includegraphics[width=0.328\textwidth]
            {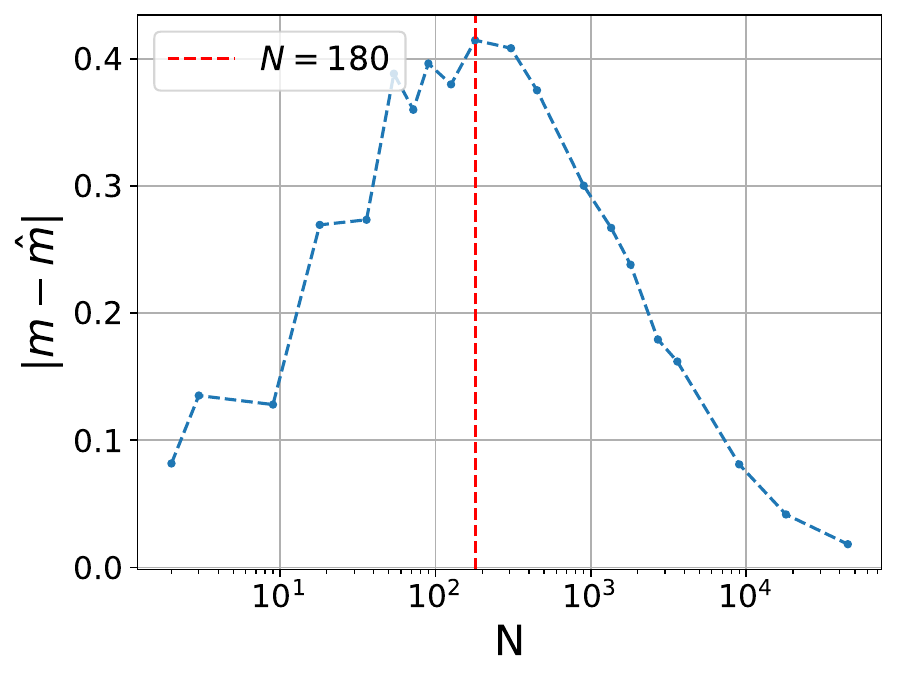}
        \includegraphics[width=0.328\textwidth]
            {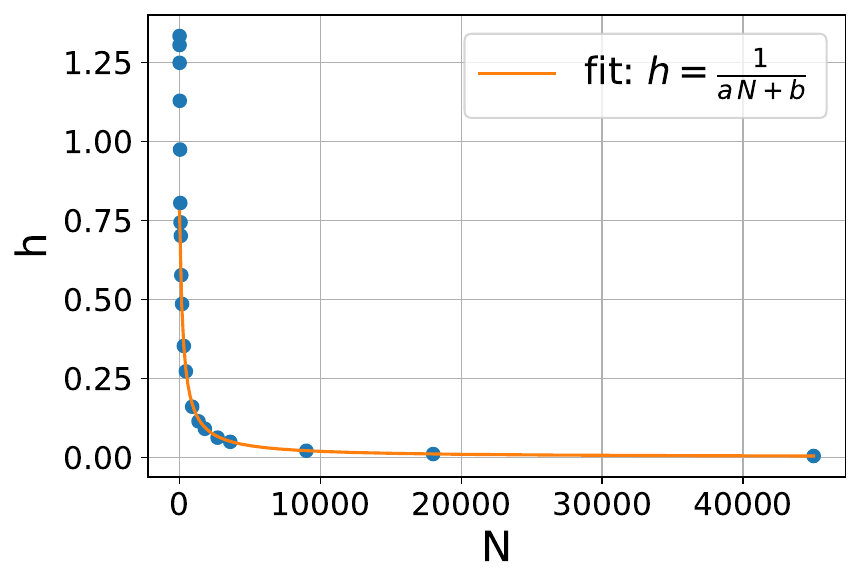}
        \includegraphics[width=0.328\textwidth]
            {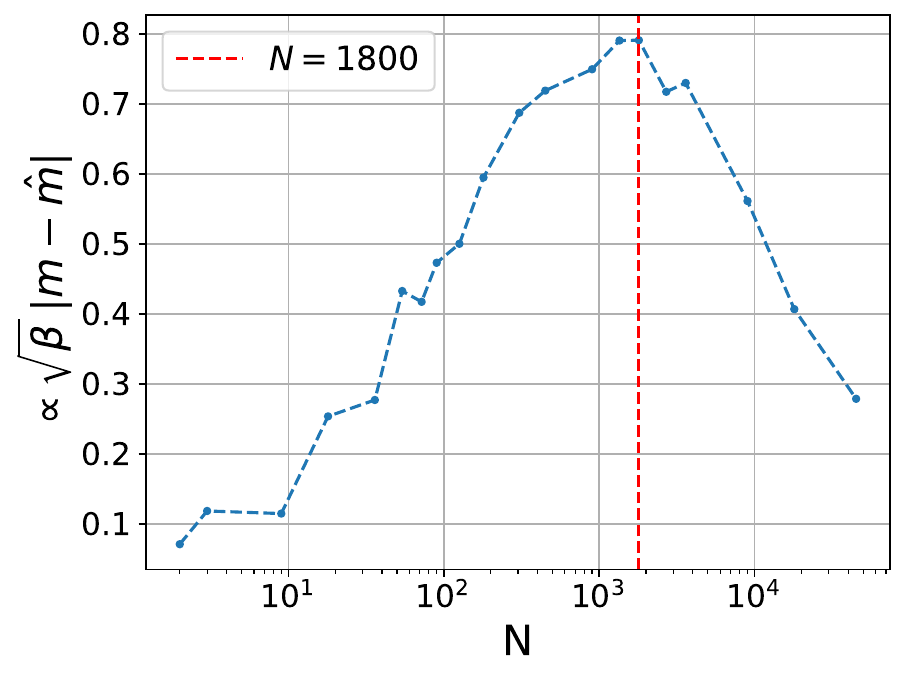}
    
        \caption{$18$-dimensional Bayesian logistic regression on the UCI SUSY dataset, with a standard Gaussian prior on the transformed coefficients.
        Left: Bell-shaped profile for the estimated error in the mean.
        Middle: Step-size scaling as a function of $N$.
        Right: Bell-shaped profile for the rescaled error in the mean.}
        \label{fig:double_descent_SUSY}
    \end{figure}

	\section{Subspace robust Wasserstein and related metrics}
	\label{sec:sliced-wasserstein}
	
	In this section, we discuss more in depth the subspace robust Wasserstein distance and compare it with some  variants.

	First, observe that the $\mswm$-distance provides rather strong control in terms of estimation of test functions.
	Indeed, by Kantorovich duality, for all $\theta \in\prosph{k}$ and Lipschitz functions $f\colon \R^k \to \R$ we have that
    	\[
	\abs*{\int f\tond*{(\scal*{\theta_i,x})_{i=1}^k} \dd \mu(x) -\int f \tond*{(\scal*{\theta_i,x})_{i=1}^k} \dd \nu(x)} \le \lip{f} \sqrt{k}\msw{\mu,\nu}.
	\]
	Furthermore, we also have for the means of the distributions that
	\begin{equation} \label{eq:control-mean-intro}
		\abs*{\int x \dd\mu-\int x \dd \nu } \leq \msw{\mu,\nu}.
	\end{equation}
	Therefore, an upper bound on $\msw{\mu_k, \pi}$ along \eqref{eq:ULA} still yields strong guarantees in terms of estimations of many important test functions, analogously to the $W_2$-metric.

	The subspace robust Wasserstein distance is related to the family of \emph{sliced Wasserstein distances}, which are transportation distances defined in terms of the Wasserstein-distance between projections on lower-dimensional subspaces.
	They were initially introduced because of their computational advantage over the classical Wasserstein metric; since then, they have found a variety of applications in different areas, particularly high-dimensional statistics and machine learning \cite{nadjahi2021sliced}, and are an object of active study \cite{bob-got-2024,car-fig-mer-wan-2025}.
	For example, the standard (squared) $2$-sliced Wasserstein distance is defined as the expected $2$-Wasserstein distance between projections on a uniform at random direction, i.e. by
	\begin{equation} \label{eq:1-sliced-wasserstein}
		\swm^2\tond*{\mu,\nu} \coloneqq {\int_{S^{d-1} } W_2^2\tond*{\mu_\theta, \nu_\theta} \dd\sigma_{d-1}(\theta)},
	\end{equation}
	where $\sigma_{d-1}$ is the uniform measure on the unit sphere $S^{d-1}$ .
	Several variants of \eqref{eq:1-sliced-wasserstein} can then be considered \cite{kit-tak-2023}, e.g. by using $p$-Wasserstein distances for different values of $1\le p \le \infty$, $k$-dimensional projections instead of just one-dimensional, or by replacing the expectation with respect to $\sigma_{d-1}$ with an $L^q\tond*{\sigma_{d-1}}$-norm.
	
	We observe, however, that the averaged counterparts like \eqref{eq:1-sliced-wasserstein} provide only weak guarantees, since they only imply that the typical $W_2$-distance between random low-dimensional pushforwards is dimension-free.   Ultimately, it is important to control the bias in any direction, not only random ones, and hence allowing practitioners to choose arbitrary low-dimensional functions whose expectation needs to be estimated; for example, in general even \eqref{eq:control-mean-intro} fails if we replace the right-hand-side with $\swm$.
	On a related note, the bias in these metrics is typically dimension-free (cf. \eqref{eq:sw2-bias-dimension-free} below). On the one hand, this means in particular that these versions are not suitable for explaining the empirical observations of Figure \ref{fig:double_descent}, as they would predict a flat curve for the iteration complexity, as opposed to the results of our  Theorem \ref{thm:informal-ridge} illustrated in Figure \ref{fig:theory}. On the other hand, as we discuss next, this fact still has interesting implications for predictive inference.
	\paragraph{Sliced Wasserstein distances and predictive inference}
	In the following lemma, we express and generalize this observation that, if we look at  $W_2$-bias in a ``typical direction'', (corresponding to the $2$-sliced Wasserstein metric $\swm$), the result is usually dimension-free 
	(cf. also \cite{DuEb24, lac-zho-2025} for related comments when projecting on a random coordinate axis instead of a random direction on the sphere).
	
	\begin{lemma} \label{lem:predictive-inference}
		Let $\theta$ be a centered random vector in $\R^d$ such that for some positive constants $M,\sigma^2$
		we have 
		$ \opnorm{\covrv{\theta}} \le M^2$ and 
		\begin{equation*}
			\pro{f(\theta) - \expe{f(\theta)} \ge t} \le e^{-\frac{t^2}{2\sigma^2\lip{f}^2}} 
		\end{equation*}
		for all Lipschitz functions $f\colon \R^d \to \R$ and $t\ge 0$. 
		Then, for $\pi,\hat \pi \in \pp_2(\R^d)$, and denoting 
		as before by $\pi_\theta, \hat \pi_\theta$ the pushforward measures with respect to $x \to \scal*{\theta,x}$, we have 
		\begin{equation}\label{eq:expected-w2-uniform-coordinate-marginal}
			\expe{W_2^2 \tond*{\pi_\theta, \hat \pi_\theta} } \le M^2\cdot{W_2^2\tond*{\pi, \hat \pi}}.
		\end{equation}
		Furthermore, 
		\begin{align*}
			\pro{W_2\tond*{\pi_\theta, {\hat \pi}_\theta} > (1+\eps)M\cdot{W_2(\pi,\hat \pi)}} & \le e^{-\frac{M^2\eps^2}{2\sigma^2}}. 
		\end{align*}    
	\end{lemma}
	
	\begin{proof}
		Let $X,\hatx\sim(\pi,\hat \pi)$ be optimally coupled in the definition of $W_2(\pi,\hat \pi)$ and set $Z\coloneqq X-\hat X$. Then, for any $\theta\in \R^d$, $\tond*{\scal*{X,\theta}, \scal*{\hat X,\theta}}$ is a coupling of $(\pi_\theta, \hat \pi_\theta)$, and so
		\begin{align*}
			\expewr{\theta}{W_2^2\tond*{\pi_\theta, {\hat \pi}_\theta}} &\leq \expewr{\theta}{\expe{{\scal{Z, \theta}^2}}}
			\\
			& = \expe{\expewr{\theta}{{\scal{Z, \theta}^2}}}
			\\
			&  \leq {M^2}\cdot\expe{{\abs*{Z}^2}}
			\\
			& = M^2\cdot {W_2^2(\pi,\hat \pi)}.
		\end{align*}
		Now consider the function $f(\theta) = \sqrt{\expe{\scal{Z,\theta}^2}}$ and observe that $\lip{f} \le W_2(\pi,\hat\pi)$.
		Indeed,
		\begin{align*}
			\abs*{f(\theta_1)-f(\theta_2)} &  = \abs*{\sqrt{\expe{\scal{Z,\theta_1}^2}}-\sqrt{\expe{\scal{Z,\theta_2}^2}}}
			\\
			& \le \sqrt{\expe{\scal*{Z,\theta_1-\theta_2}^2}}
			\\
			& \leq \abs*{\theta_1-\theta_2}\sqrt{\expe{\abs*{Z}^2}}
			\\
			& = \abs*{\theta_1-\theta_2}W_2(\pi,\hat \pi).
		\end{align*}
		By the Lipschitz-concentration assumption, it follows that for $t\ge 0$
		\begin{align*}
			\pro{W_2\tond*{\pi_\theta, {\hat \pi}_\theta} > M\cdot{W_2(\pi,\hat \pi)}+t} & \le \pro{f(\theta) \ge M\cdot {W_2(\pi,\hat \pi)}+ t}
			\\
			& \le \pro{f(\theta) \ge \expe{f(\theta)}+t} 
			\\
			& \le \expo{-\frac{t^2}{2\sigma^2W_2^2\tond*{\pi,\hat \pi}}}.
		\end{align*}
		Choosing $t = \eps M\cdot{W_2(\pi,\hat\pi)}$ gives the desired conclusion.
	\end{proof}
	
	\begin{example}\label{cor:typical-sliced-bias}
		\begin{itemize}
			\item Suppose that $\theta\sim \unif{S^{d-1}}$ is distributed uniformly on the unit sphere. Then
			\begin{align*}
				&\swm^2\tond*{\pi,\hat\pi} =\expe{W_2^2\tond*{\pi_\theta, {\hat \pi}_\theta}} \leq \frac{W_2^2(\pi,\hat \pi)}{d}
				\\
				&\pro{W_2\tond*{\pi_\theta, {\hat \pi}_\theta} > (1+\eps)\frac{W_2(\pi,\hat \pi)}{\sqrt{d}}}  \lesssim  e^{-c\eps^2}
			\end{align*}
			where $c>0$ is an absolute constant.
			\item Suppose that $\theta\sim \mathcal{N}(0,\Sigma)$ where $\Sigma \preccurlyeq \frac1\alpha I_d$. Then,
			\begin{align*}
				& \expe{W_2^2\tond*{\pi_\theta, {\hat \pi}_\theta}} \leq \frac1{\alpha} {W_2^2(\pi,\hat \pi)},
				\\
				&\pro{W_2\tond*{\pi_\theta, {\hat \pi}_\theta} > (1+\eps)\frac1{\sqrt{\alpha}}{W_2(\pi,\hat \pi)}}  \le e^{-\frac{\eps^2}{2}}.
			\end{align*}
		\end{itemize}
	\end{example}
	Notice that combining the first bound in the above example with \eqref{eq:w2-bias-strong-log-conc} we deduce that under Assumption \ref{ass:strong-log-conc-log-smooth} we  have that for $h\leq \frac1\beta$ 
	\begin{equation} \label{eq:sw2-bias-dimension-free}
		\swm^2\tond*{\pi,\hat\pi} \lesssim \kappa \beta h^2,
	\end{equation}
	i.e. the sliced-Wasserstein bias is always dimension-free.
	More generally, Lemma \ref{lem:predictive-inference} has interesting consequences for predictive inference \cite{fearnhead2024scalable}.
	Consider the situation where $f\colon \R\to \R$ is a $1$-dimensional $1$-Lipschitz function and $a \sim \nu$ is a random variable satisfying the assumption of Lemma \ref{lem:predictive-inference}, which plays the role of an unseen datapoint for which you want to do predictive inference: given the normalized random vector $\tilde a = \frac{a}{\sqrt{d\opnorm{\Sigma}}}$, for $\Sigma = \covrv{a}$, we wish to estimate the expectation $\expewr{\pi}{f\scal*{x,\tilde a}}$. Then, by Kantorovich duality 
	\begin{equation*}
		\begin{split}
			\abs*{\expewr{\pi}{f\scal*{x,\tilde a}}-\expewr{\hat \pi}{f\scal*{x, \tilde a}}} & \le \frac{W_2(\pi_a,\hat \pi_a)}{\sqrt{d\opnorm{\Sigma}}}.
		\end{split}
	\end{equation*}
	Lemma \ref{lem:predictive-inference} with $M^2 = \opnorm{\Sigma}$ together with \eqref{eq:w2-bias-strong-log-conc} then implies with high-probability that
	\[
	\abs*{\expewr{\pi}{f\scal*{x,\tilde a}}-\expewr{\hat \pi}{f\scal*{x,\tilde a}}} \lesssim h\sqrt{\kappa \beta},
	\]
	i.e., for typical datapoints we obtain a dimension-free error, as opposed to what happens when considering the worst-possible test-direction.

	\section{Random tensors in high dimension}
	\label{sec:app:random-tensors-high-dimensions}
	
	In this section, we prove several high-probability bounds and related technical results needed in Section \ref{sec:appl-stats-models}.
	For a random variable/vector $X$, we denote by $\sgnorm{X}$
    its sub-Gaussian norm,
    as defined in \cite{ver-2026}.

	\begin{lemma}\label{lem:tail-prob-bound-opnom-sample-cov}
		Under Assumption \ref{ass:ridge-separable-assumption}, there exists an absolute constant $C>0$ such that, for all $u\ge 0$, 
		if  $N\ge d+u$ then with probability at least $1-2e^{-u}$ we have that 
		\begin{align}
			\label{eq:upp-bound-tail-prob-norm-sample-cov}
			{{\sum_{i=1}^N a_i a_i^T} \preccurlyeq N\opnorm{\Sigma}\tond*{1+K^2\sqrt{\frac{C\left(d+u\right)}{N}}}I_d},
			\\ \label{eq:low-bound-in-prob-sample-cov}
			{\sum_{i=1}^N a_i a_i^T \succcurlyeq N \frac{\opnorm{\Sigma}}{\condcov}\tond*{1-K^2\sqrt{\frac{C
							\left(d+u\right)}{N}}} I_d}. 
		\end{align}
	\end{lemma}

    \begin{proof}[Proof of Lemma \ref{lem:tail-prob-bound-opnom-sample-cov} and Lemma \ref{lem:tail-prob-bound-opnom-sample-cov-2}]
    Let us first put the vectors in isotropic position, i.e. consider the random vectors $\Sigma^{-\frac12} a_i$: these are centered and with covariance equal to $I_d$. Moreover, by Assumption \ref{ass:ridge-separable-assumption} we have for every $v\in \R^d$ that
    \[
        \sgnorm*{\scal*{\Sigma^{-\frac12}a_i, v}} \le K \sqrt{\scal*{\Sigma^{-\frac12}v, \Sigma \Sigma^{-\frac12} v}} = K \abs*{v}.
    \]
    Applying \cite[Remark 4.7.3]{ver-2026}, it follows that for every $u\ge 0$ with probability at least $1-2e^{-u}$ we have
    \begin{equation}
        \opnorm{\Sigma^{-1/2}\tond*{\sum_{i=1}^N a_i a_i^T - N \Sigma} \Sigma^{-1/2}} \lesssim K^2 \tond*{\sqrt{N(d+u)}+d+u}.
    \end{equation}
    It follows that, for an absolute constant $C>0$,
    \begin{equation}            
    N \Sigma \tond*{1-CK^2 \tond*{\sqrt{\frac{d+u}{N}}+\frac{d+u}{N}}}
            \preceq \sum_{i=1}^N a_i a_i^T 
            \preceq N \Sigma \tond*{1+CK^2 \tond*{\sqrt{\frac{d+u}{N}}+\frac{d+u}{N}}}.
            \end{equation}
    The desired conclusions follow, after increasing $C$, using that
    $\frac{\opnorm{\Sigma}}{\kappa_\Sigma} I_d \preceq \Sigma \preceq \opnorm{\Sigma}I_d$
    and that $\frac{d+u}{N}\le \sqrt{\frac{d+u}{N}}$ when $N\ge d+u$, with the reverse inequality when $N\le d+u$.
    Here the lower comparison is used only when its coefficient is nonnegative; otherwise, the claimed lower bound follows from the positive semidefiniteness of $\sum_{i=1}^N a_i a_i^T$.
\end{proof}

	\begin{lemma}\label{lem:chevet_tensor}
		Under Assumption \ref{ass:ridge-separable-assumption}, there exists an absolute constant $C>0$ such that, for every $u\ge 0$, we have that
		\[
		\pro{\sup_{\abs*{\theta}_{2} \leq 1} \sum^{N}_{i=1} \abs*{\scal*{a_i,\theta}}^3 \geq C(K\|\Sigma\|^{1/2}_{\textnormal{op}})^{3}\left(d^{1/2} + N^{1/3} + \sqrt{u}\right)^{3}} \leq 2e^{-u}.
		\]
	\end{lemma} 
	
	\begin{proof}
		In this proof, for a vector $v\in \R^d$ and $p\ge 1 $ we denote by $\abs*{v}_p = \tond*{\sum_{i=1}^d \abs*{v_i}^p}^{\frac1p}$ the $\ell^p$-norm of $v$.
		Let $A \in \mathbb{R}^{N \times d}$ be the random matrix whose $i$-th row is $a_i$. Then
		\begin{align*}
			\sup_{\abs*{\theta}_2 \leq 1} \sum^{N}_{i=1} \abs*{\left\langle a_i,\theta \right\rangle}^{3} = \sup_{\abs*{\theta}_2 \leq 1}  \abs*{A\theta}^3_{3} = \|A\|^{3}_{2 \to 3}.
		\end{align*}
		It suffices to control $\|A\|^{3}_{2 \to 3}$. Note that
		\begin{align*}
			\|A\|_{2 \to 3} 
			= \sup_{\abs*{x}_2 \le 1, \ \abs*{y}_{3/2} \le 1} \scal*{Ax, y} = \sup_{U \in \U} \scal*{A,U}_{\mrmf},
		\end{align*}
		where $\mathcal{U}:=\{yx^\top:\ \abs*{x}_2\le 1,\ \abs*{y}_{3/2}\le 1\} \subset \mathbb R^{N\times d}$.
		
		Define the process $X_U\coloneqq \scal*{A,U}_{\mathrm{F}}$ on $\U$. We verify that it is a sub-Gaussian process as follows.
		Given $U, V \in \U$, we have
		\begin{align*}
			& \sgnorm*{\scal*{A,U}_\mrmf-\scal*{A,V}_\mrmf}
			\\
			= &\sgnorm*{\sum_{i=1}^N \scal*{a_i, U_{i,\cdot}-V_{i,\cdot}}}
			\\
			\overset{(i)}{\lesssim}  & \sqrt{\sum_{i=1}^N \sgnorm*{\scal*{a_i, U_{i,\cdot}-V_{i,\cdot} }}^2}
			\\
			\overset{(ii)}{\leq}  & K\sqrt{\opnorm{\Sigma}}\sqrt{\sum_{i=1}^N  \abs*{U_{i,\cdot}-V_{i,\cdot} }^2}
			\\
			= &   K\sqrt{\opnorm{\Sigma}} \fronorm{U-V}.
		\end{align*}
		In the above, (i) follows from the sub-Gaussian norm of sums of independent centered sub-Gaussian random variables \cite[Prop. 2.7.1]{ver-2026}. (ii) follows since by Assumption \ref{ass:ridge-separable-assumption}
		for any $v\in\mathbb R^d$, $\sgnorm*{\scal*{a_i,v}} \le K\sqrt{v^\top \Sigma v} \le K \sqrt{\opnorm{\Sigma}} \abs*{v}_2$.
		
		By the high-probability generic chaining bound for sub-Gaussian processes (see \cite[Theorem. 8.5.2 and Exercise 8.37(b)]{ver-2026}), for $t > 0$, 
		\begin{align} \label{eq:norm_A-23-temp}
			\sup_{U \in \U} X_U \lesssim K\sqrt{\opnorm{\Sigma} }\tond*{w(\U)+t\rad(\U)}
		\end{align}
		with probability at least \(1-2e^{-t^2}\), where $w$ is the Gaussian width and $\rad$ is the radius of a set. It remains to estimate $w(\U)$ and $\rad(\U)$.
		For the second, observe that 
		\begin{align}\label{eq:radU}
			\rad(\U) \coloneqq \sup_{U \in \U} \fronorm{U} = \sup_{x\in S, y\in T } \fronorm{y x^T} \le \sup_{x\in S, y\in T }\abs*{x}\abs*{y} \le 1.
		\end{align}
		For the Gaussian width, let $G$ be an $\R^{N\times d}$ random matrix with i.i.d. standard Gaussian entries. Then, 
		\begin{align*}
			w(\U) &\coloneqq \expe{\sup_{U \in \U} \scal*{G,U}_\mrmf}
			\\
			& =  \expe{\sup_{x\in S, y \in T} \scal*{Gx,y}}
			\\
			& \lesssim \tond*{w(S)\rad(T) +w(T)\rad(S)}.
		\end{align*}
		where the last inequality follows from \cite[Theorem 8.6.1]{ver-2026}.
		As before $\rad(S) =1, \rad(T)\le 1$, and furthermore with $g_d$ and $g_N$ standard Gaussians in $\R^d$ and $\R^N$ respectively
		\begin{align*}
			w(S) & =  \expe{\sup_{x\in S} \scal*{x,g_d}} \leq \expe{\abs*{g_d}}  \lesssim \sqrt{d}
			\\
			w(T) & = \expe{\sup_{y\in T} \scal*{y,g_N}} \leq \expe{\abs*{g_N}_{\ell^3}} \lesssim N^{\frac{1}{3}}.
		\end{align*}
		Hence, 
		we find that
		\begin{equation}
			w(\U) \lesssim \sqrt{d}+N^{\frac{1}{3}}.
		\end{equation}
		Inserting this and \eqref{eq:radU} in \eqref{eq:norm_A-23-temp} gives that
		for every $t\ge 0$
		\[
		\norm*{A}_{2\to 3} \lesssim K\sqrt{\opnorm{\Sigma}}\tond*{\sqrt{d}+N^{\frac13}+t}
		\]
		with probability at least $1-2e^{-t^2}$, from which the conclusion easily follows choosing $t=\sqrt{u}$.
		
	\end{proof}

	\begin{proof}[Proof of Lemma \ref{lem:high_prob_cond_scal}]
		Throughout this proof, $C>0$ denotes a large enough absolute constant whose value can possibly be increased from line to line.  We prove more generally that, for all $u\ge 0$,  if $N> CK^{4}\left(d+u\right)$ then 
		\begin{equation}
			\pro{\hess V \succcurlyeq \bar \alpha N \frac{\opnorm{\Sigma}}{\condcov}\tond*{1-K^2\sqrt{\frac{C\left(d+u\right)}{N}}} I_d} \ge 1-2\expo{-u},
		\end{equation}
		\begin{equation}\label{eq:cond-numb-prob-bound}
			\pro{\kappa \le \frac{\bar \beta}{\bar \alpha} \frac{\condcov\tond*{1+K^2\sqrt{\frac{C\left(d+u\right)}{N}}}}{1-K^2\sqrt{\frac{C\left(d+u\right)}{N}} } } \ge 1-2\expo{-u}
		\end{equation}
		and
		\begin{equation}\label{eq:scale-numb-prob-bound}
			\pro{\scalenumber \leq C \frac{\bar M}{\bar \alpha^{\frac{3}{2}}} \frac{K^3\condcov^{3/2}}{\tond*{1-K^2\sqrt{\frac{C\left(d+u\right)}{N}}}^{3/2}} \left(\left(\frac{d + u}{N}\right)^{3/2} + \frac{1}{\sqrt{N}} \right)} \ge 1-4\expo{-u}.
		\end{equation}
        Lemma \ref{lem:high_prob_cond_scal} follows then easily from choosing $u=d$ and using $N\ge d^{3/2}$.

		To prove the claim, notice first that $\hess V(x) = \lambda I_d + \sum_{i=1}^N u_i''\tond*{\scal*{x,a_i}}a_ia_i^T$ and so
		\[
		\lambda I_d+ \bar \alpha \sum_{i=1}^N a_i a_i^T \preccurlyeq \hess V \preccurlyeq \lambda I_d + \bar \beta \sum_{i=1}^N a_i a_i^T.
		\]
		By Lemma \ref{lem:tail-prob-bound-opnom-sample-cov},
        with probability at least $1-2e^{-u}$, we have
		\begin{align}
			\alpha \geq \bar \alpha N \frac{\opnorm{\Sigma}}{\condcov}\tond*{1-K^2\sqrt{\frac{C\left(d+u\right)}{N}}}
		\end{align}
		and
		\begin{align}
			\opnorm{\sum a_ia_i^T}  \leq N\opnorm{\Sigma}\tond*{1+K^2\sqrt{\frac{C\left(d+u\right)}{N}}}. 
		\end{align}         
		and the first bound \eqref{eq:cond-numb-prob-bound} immediately follows.
		Moving to \eqref{eq:scale-numb-prob-bound}, 
		considering ridge-separable $V$, we have that the third derivative is a symmetric $3$-tensor with the following operator norm
		\begin{align*}
			\|\nabla^{3}V(x)\|_{\textnormal{op}} &= \sup_{\|u_{1}\| \leq 1,\|u_{2}\| \leq 1,\|u_{3}\| \leq 1}\sum^{N}_{i=1}u_i'''(\left\langle x,a_{i}\right\rangle) \left\langle a_{i},u_{1}\right\rangle \left\langle a_{i},u_{2}\right\rangle \left\langle a_{i},u_{3}\right\rangle\\
			&\leq  \sup_{\|u_{1}\| \leq 1,\|u_{2}\| \leq 1,\|u_{3}\| \leq 1}\overline{M}\sum^{N}_{i=1}\frac{1}{3}\left(\left|\left\langle a_{i},u_{1}\right\rangle\right|^{3} + \left|\left\langle a_{i},u_{2}\right\rangle\right|^{3} +\left|\left\langle a_{i},u_{3}\right\rangle\right|^{3}\right)\\
			&= \sup_{\|u\| \leq 1}\overline{M}\sum^{N}_{i=1}\left|\left\langle a_{i},u\right\rangle\right|^{3},
		\end{align*}
		and by Lemma \ref{lem:chevet_tensor} we have with probability $1-2e^{-u}$ that
		\begin{equation}
			\|\nabla^{3}V(x)\|_{\textnormal{op}} \leq C\overline{M}(K\|\Sigma\|^{1/2}_{\textnormal{op}})^{3}\left(d^{1/2} + N^{1/3} + \sqrt{u}\right)^{3}.
		\end{equation}
		By a union bound, putting everything together, with probability at least $1-4e^{-u}$
		\begin{align}
			\alpha & \geq \bar \alpha N \frac{\opnorm{\Sigma}}{\condcov}\tond*{1-K^2\sqrt{\frac{C\left(d+u\right)}{N}}}
		\end{align}
		and 
		\begin{equation}
			\begin{split}
				M & \leq C\overline{M}(K\|\Sigma\|^{1/2}_{\textnormal{op}})^{3}\left(d^{1/2} + N^{1/3} + \sqrt{u}\right)^{3}          
			\end{split}
		\end{equation}
		and so 
		\begin{equation}
			\begin{split}
				\scalenumber \leq C \frac{\bar M}{\bar \alpha^{\frac{3}{2}}} \frac{K^3\condcov^{3/2}}{\tond*{1-K^2\sqrt{\frac{C\left(d+u\right)}{N}}}^{3/2}} \left(\left(\frac{d + u}{N}\right)^{3/2} + \frac{1}{\sqrt{N}} \right),
			\end{split}
		\end{equation}
		as required.
		
	\end{proof}

	\section{Auxiliary results for LD and ULA}
	\label{sec:app:auxiliary-ld-ula}
	
	\subsection{Invariance under orthogonal transformations}
	
	Both LD and ULA are independent of the coordinate system: heuristically, this is because (i) gradient flow seeks the steepest decent direction, which is independent of a coordinate system, and (ii) Brownian motion is spherically symmetric.
	Alternatively, we can proceed as follows.
	Let $\pi \propto e^{-V} \in \pp(\R^d)$ 
	and consider the corresponding Langevin dynamics
	\[
	X_0 = 0, \quad dX_t  =-\nabla V(X_t) dt + \sqrt{2}dB_t.
	\]
	Let $O\colon \R^d\to \R^d$ be an orthogonal matrix and consider the corresponding pushforward measure, i.e.,
	\begin{align}
		\pi_O &= O\# \pi,
		\\ \label{eq:change-density-orthogonal-transformation}
		\pi_O(x) &= \pi(O^*x) \propto e^{-V(O^*x)},
		\\
		X\sim \pi & \implies OX \sim \pi_O.
	\end{align}
	Then, if
	\[
	X_0 =0,\quad dX_t = -\nabla V(X_t) dt + \sqrt{2}dB_t
	\]
	is a Langevin dynamics for $\pi$ and 
	\[
	Y_0 = 0, \quad dY_t = -O\nabla V(O^*Y_t)dt  +\sqrt{2}dB_t
	\]
	for $\pi_O$, we have that
	\[
	Y_t \eqlaw OX_t.
	\]
	Indeed,
	\begin{align*}
		dOX_t &= - O\nabla V(X_t)dt + \sqrt{2}OdB_t
		\\
		& = -O \nabla V(O^*OX_t)dt +\sqrt{2}d\tilde B_t.
	\end{align*}
	since $OO^*=I_d$ and Brownian motion is rotationally invariant.
	Similar considerations holds for ULA, and in particular we can show that for the biased limits we have the relation 
	\begin{equation} \label{eq:biased-limit-orthogonal-pushforward}
		\widehat{\pi_O} = O \#\hat \pi.
	\end{equation}

    \subsection{Classical bounds}
	
	\begin{proof}[Proof of Lemma \ref{lem:moment_bounds}]
		For the first bound in \eqref{eq:moment-bounds}, see \cite[Lemma 2]{dalalyan2022bounding}.
        The second and third bounds follows from subgaussianity of the score $\nabla V$ under $\pi$, cf. \cite[Lemma 6.2.7]{che-book-2024+}, \cite{alt-che-2026-shifted2}.
        To prove \eqref{eq:Q-bound}, observe that
        \begin{align*}
            \opnorm{Q(x)} & = \opnorm{\hess V(x) - \expewr{X\sim \pi}{\hess V(X)}}
            \\
            & \le \expewr{X\sim \pi}{\opnorm{\hess V(X)-\hess V(x)}}
            \\
            & \le M \expe{\abs*{X-x^*}+\abs*{x-x^*}}
            \\
            &\le M \tond*{\sqrt{\frac{d}{\alpha}}+\abs*{x-x^*}}.
        \end{align*}

	\end{proof}

	\section{Remaining omitted proofs}
	\label{sec:omitted-proofs}

    \begin{proof}[Proof of Proposition \ref{pro:msw-bias-gaussian}]
		Recall that
		$ \pi \propto e^{-V}$ with 
		\[
		V(x) = \frac{1}{2}\scal*{x-m,A(x-m)}
		\]
		and $ \alpha I_d \preccurlyeq A \preccurlyeq \beta I_d$. For $h\le \frac1\beta$, $\hat \pi$ admits the known explicit form
		$\hat \pi = \mathcal{N}\tond*{m, A^{-1}\tond*{I_d - \frac{h}{2}A}^{-1}}$,
		from which we see in particular that $\hat m = m$.
		By considering an orthogonal transformation (cf. Section \ref{sec:app:auxiliary-ld-ula}), we can assume without loss of generality  that $A$ is diagonal, so that both $\pi$ and $\hat \pi$ are product measures.
		The lower bound on $W_2(\pi,\hat \pi)$ then follows easily, since for the $2$-Wasserstein distance between one-dimensional Gaussians with the same mean we have
		\[
		W_2\tond*{\mathcal{N}(x,\sigma_1^2),\mathcal{N}(x,\sigma_2^2)} =\abs*{\sigma_1-\sigma_2}.
		\]
		To prove the upper bound on $\msw{\pi,\hat \pi}$, 
		by  \eqref{eq:general-bounds-mswm} it suffices to construct a coupling $\tond*{X,\hat X}$ of $\tond*{\pi,\hat\pi}$ such that $\left\|\covrv{X-\hat X}\right\|_{\textnormal{op}}\lesssim \beta h^{2}$.
		By coupling all coordinates independently, it suffices to prove the claim in dimension $d=1$, because then $\covrv{X-\hat X}$ will be diagonal and its operator norm equal to maximum diagonal entry.
		The one-dimensional case is straightforward to check, remembering that the variance $\sigma^2$ satisfies $\sigma^2 \ge \frac1\beta$, and since an appropriate coupling 
		$X\sim \mathcal{N}(m,\sigma^2)$, $\hat X\sim \mathcal{N}\tond*{m, \sigma^2 \frac{2\sigma^2}{2\sigma^2-h}}$ can be constructed by setting 
		\[
		\hat X = \sqrt{\frac{2\sigma^2}{2\sigma^2-h}}\tond*{X-m}+m.
		\]    
	\end{proof}

    \begin{proof}[Proof of Lemma \ref{lem:matrix-estimates-ula}]
        For every $j\geq1$,
		\[
		\begin{aligned}
			\|A^{3/2}\tond*{I_d-hA}^j\|_{\mathrm{op}}&=\max_{a\in\sigma(A)}a^{3/2}(1-ha)^j \leq \sup_{\alpha\leq a\leq\beta}a^{3/2}e^{-hja} \lesssim \min\!\left\{\beta^{3/2},(hj)^{-3/2}\right\},
		\end{aligned}
		\]
		where we used $0\leq ha\leq1$, the inequality
		$1-ha\leq e^{-ha}$, and
		\[
		\sup_{a\geq0}a^{3/2}e^{-hja}
		=(hj)^{-3/2}\sup_{x\geq0}x^{3/2}e^{-x}
		\lesssim (hj)^{-3/2}.
		\]
		Let $N:=\lceil(h\beta)^{-1}\rceil$, then since $h\beta\leq1$, we have $N\lesssim(h\beta)^{-1}$ and
		\[
		\begin{aligned}
			\sum_{j=1}^{\infty}\|A^{3/2}\tond*{I_d-hA}^j\|_{\mathrm{op}}
			&\lesssim
			N\beta^{3/2}
			+h^{-3/2}\sum_{j=N+1}^{\infty}j^{-3/2} \\
			&\lesssim
			N\beta^{3/2}+h^{-3/2}N^{-1/2}
			\lesssim \frac{\sqrt{\beta}}{h}.
		\end{aligned}
		\]
		Moreover, the $j=0$ term satisfies
		$
		\|A^{3/2}\|_{\mathrm{op}}
		\leq\beta^{3/2}
		\leq\frac{\sqrt{\beta}}{h},
		$
		because $h\beta\leq1$. Multiplying the resulting bound by $h^2$ proves \eqref{eq:smoothing}.
		To prove \eqref{eq:gaussian-sum}, consider an eigenvalue  $a\in[\alpha,\beta]$ of $A$. We have that $0\leq 1-ha<1$, and hence
		\[
		\begin{aligned}
			h^3\sum_{j=0}^{\infty}a^2(1-ha)^{2j}
			&=\frac{h^3a^2}{1-(1-ha)^2}  =h^2\frac{a}{2-ha}
			\leq h^2a,
		\end{aligned}
		\]
		Applying this inequality to every eigenvalue of $A$ yields \eqref{eq:gaussian-sum}.
    \end{proof}

		\begin{proof}[Proof of Lemma \ref{lem:one-dimensional-potential}]
		For $(\delta,b)$ near
		$(0,0)$, set
		\[
		F(\delta,b):=\int_{\mathbb R} x\,e^{-x^2/2+\delta\sin x-bx}\,dx.
		\]
		By dominated convergence, $F$ is $C^1$ in a neighborhood of $(0,0)$, with
		\[
		\partial_b F(\delta,b)
		=
		-\int_{\mathbb R} x^2 e^{-x^2/2+\delta\sin x-bx}\,dx.
		\]
		Since
		\[
		F(0,0)=0,
		\qquad
		\partial_bF(0,0)=-\int_{\mathbb R}x^2e^{-x^2/2}\,dx<0,
		\]
		the implicit function theorem yields $\eta>0$ and a $C^1$ map $b:(-\eta,\eta)\to\mathbb R$ such that $b(0)=0$ and
		\[
		F(\delta,b(\delta))=0,\qquad |\delta|<\eta.
		\]
		Now define for $\delta \in (-\eta,\eta)$ the function
		\[
		U_\delta(x):=\frac{x^2}{2}-\delta\sin x+b(\delta)x
		\]
		so that $\int x e^{-U_\delta(x)} dx =0$ and for all $\delta \in (-\eta,\eta)$ and $x\in \R$
		\[
		U_\delta''(x)=1+\delta\sin x.
		\]
		By possibly restricting $\eta$, we can then also assume without loss of generality that
		\[
		\frac12 \le U_\delta'' \le 2.
		\]
		Finally, set
		\[
		I(\delta)\coloneqq\int_{\mathbb R} x\,U_\delta''(x)e^{-U_\delta(x)}\,dx
		=
		\delta\int_{\mathbb R}x\sin x\,e^{-x^2/2+\delta\sin x-b(\delta)x}\,dx.
		\]
		Since $b(\delta)\to 0$ as $\delta\to0$, by dominated convergence we have that
		\[
		\lim_{\delta \to 0}\frac{I(\delta)}{\delta}
		=
		\int_{\mathbb R}x\sin x\,e^{-x^2/2}\,dx =\sqrt{2\pi}e^{-1/2}>0.
		\]
		Therefore $I(\bar \delta)>0$ for all sufficiently small $ \bar \delta>0$.
		We can now set $C_1 = -\bar \delta , C_2 = b(\bar \delta) $ and $C_3 = \log \tond*{\int_\R \expo{-U_{\bar\delta}}}$ in the definition of $U$ and $\rho = e^{-U}$. By construction $\rho$ is a centered probability measure with $\frac12 \le U'' \le 2$ and 
		\[
		\int_\R x U''(x) \rho(x) dx >0,
		\]
		as desired.
		The last condition in the lemma follows from \cite[Proposition 3.1]{che-che-wee-wea-2024} choosing $f(x) =x$.
	\end{proof}

	\begin{proof}[Proof of Lemma \ref{lem:tensorization-mswm}]
		For $i=1,2$, let $(X_i,Y_i)$ be an optimal coupling in the definition of
		$\mswm(\mu_i,\nu_i)$, whose existence follows from
		\cite[Lemma 1]{paty2019subspace}; without loss of generality, we assume that the two couplings are
		independent. Notice that
		$
		\tond*{(X_1,X_2),(Y_1,Y_2)}
		$
		is a coupling of $(\mu,\nu)$.
		
		For $u\in\R^{d_1}$ and $v\in\R^{d_2}$, independence and the
		Cauchy--Schwarz inequality give
		\begin{align*}
			\expe{
				\left(
				\scal*{u,X_1-Y_1}
				+
				\scal*{v,X_2-Y_2}
				\right)^2
			}
			&\le
			\mswm^2\tond*{\mu_1,\nu_1}\abs*{u}^2
			+
			\mswm^2\tond*{\mu_2,\nu_2}\abs*{v}^2
			\\
			&+
			2\mswm\tond*{\mu_1,\nu_1}
			\mswm\tond*{\mu_2,\nu_2}
			\abs*{u}\abs*{v}
			\\
			&\le
			\left(
			\mswm^2\tond*{\mu_1,\nu_1}
			+
			\mswm^2\tond*{\mu_2,\nu_2}
			\right)
			\left(
			\abs*{u}^2+\abs*{v}^2
			\right).
		\end{align*}
		Taking the supremum over $\abs*{u}^2+\abs*{v}^2=1$ and using the
		definition of $\mswm$ yields
		\[
		\mswm^2\tond*{\mu,\nu}
		\le
		\mswm^2\tond*{\mu_1,\nu_1}
		+
		\mswm^2\tond*{\mu_2,\nu_2},
		\]
		as desired.
	\end{proof}

\end{document}

%% file: commands.tex
\theoremstyle{plain}
\newtheorem{theorem}{Theorem}
\newtheorem{proposition}{Proposition}[section]
\newtheorem{lemma}{Lemma}[section]

\theoremstyle{remark}
\newtheorem{remark}{Remark}[section]
\newtheorem{example}{Example}[section]

\theoremstyle{definition}

\newtheorem{assumption}{Assumption}[section]

\newcommand{\numberset}{\mathbb}
\newcommand{\N}{\numberset{N}}

\newcommand{\R}{\numberset{R}}

\DeclarePairedDelimiter{\norm}{\lVert}{\rVert}
\DeclarePairedDelimiter{\abs}{\lvert}{\rvert}

\DeclarePairedDelimiter{\tond}{(}{)} 
\DeclarePairedDelimiter{\quadr}{[}{]}
\DeclarePairedDelimiter{\graf}{\{}{\}} 
\DeclarePairedDelimiter{\scal}{\langle}{\rangle} 

\DeclarePairedDelimiter{\lip}{\lVert}{\rVert_{\mathrm{Lip}}}
\DeclarePairedDelimiter{\sgnorm}{\lVert}{\rVert_{\psi_2}}

\DeclareMathOperator*{\argmin}{arg\,min}

\DeclareMathOperator{\law}{law}

\newcommand{\eps}{\varepsilon}

\newcommand{\hess}{\nabla^2}

\newcommand{\expo}[1]{\exp\tond*{#1}}

\newcommand{\prob}{\mathds{P}}
\newcommand{\unif}[1]{\mathrm{Unif}\tond*{#1}}

\newcommand{\pro}[1]{\prob \tond*{#1}}
\newcommand{\expe}[1]{\mathbb{E}\quadr*{#1}}
\newcommand{\expewr}[2]{\mathbb{E}_{#1}\quadr*{#2}}

\newcommand{\cova}[1]{\mathrm{Cov}\quadr*{#1}}
\newcommand{\covrv}[1]{\mathrm{Cov}\quadr*{#1}}

\newcommand{\msw}[1]{\mathrm{SRW}\tond*{#1}}
\newcommand{\mswm}{\mathrm{SRW}}

\newcommand{\swm}{\mathrm{SW}_2}

\newcommand{\pp}{\mathcal{P}}

\newcommand{\eqlaw}{\stackrel{d}{=}}
\newcommand{\iid}{\mathrel{\overset{\text{i.i.d.}}{\sim}}}

\newcommand{\hatx}{\hat{X}}

\newcommand{\dd}{\, d}

\newcommand{\prosph}[1]{\tond*{S^{d-1}}^{#1}}

\newcommand{\opnorm}[1]{\norm*{#1}_{\mathrm{op}}}
\newcommand{\fronorm}[1]{\norm*{#1}_{\mathrm{F}}}
\newcommand{\U}{\mathcal{U}}
\newcommand{\mrmf}{\mathrm{F}}

\DeclareMathOperator{\rad}{rad}

\newcommand{\scalenumber}{\mathfrak{S}}

\newcommand{\trace}[1]{\mathrm{Tr}\tond*{#1}}

\newcommand{\condcov}{\kappa_{\Sigma}}